\documentclass[twoside]{article}

\usepackage[accepted]{aistats2025}
\usepackage[round]{natbib}

\usepackage[utf8]{inputenc} % allow utf-8 input
\usepackage[T1]{fontenc}    % use 8-bit T1 fonts

\usepackage{microtype}
\usepackage{graphicx}
\usepackage{booktabs} % for professional tables
\usepackage{xcolor}         % colors
\usepackage{url}            % simple URL typesetting
\usepackage{amsfonts}       % blackboard math symbols
\usepackage{nicefrac}       % compact symbols for 1/2, etc.
\usepackage{subcaption}
\usepackage{placeins}
\usepackage{siunitx}
\usepackage{xurl}
\usepackage{hyperref}

\usepackage{multirow}
\usepackage{tikz}
\usetikzlibrary{shapes,arrows.meta,positioning,backgrounds,fit, graphs}

\usepackage{amsmath}
\usepackage{amssymb}
\usepackage{amsthm}
\usepackage{thm-restate}

\declaretheorem[name=Theorem, numberwithin=section]{theorem}
\declaretheorem[name=Lemma, sibling=theorem]{lemma}

\declaretheorem[name=Corollary, sibling=theorem]{corollary}

\newcommand{\calx}{\mathcal{X}}

\newcommand{\CP}{\mathcal{E}}
\newcommand{\VP}{\mathcal{V}}

\newcommand{\R}{\mathbb{R}}
\newcommand{\E}{\operatorname*{\mathbb{E}}}

\DeclareMathOperator*{\Var}{Var}
\DeclareMathOperator*{\Cov}{Cov}
\DeclareMathOperator*{\argmin}{argmin}

\newcommand{\sdp}{\tilde{s}}

\newcommand{\adultdelta}{\approx 4.7\cdot 10^{-7}}

\newcommand{\MSE}{\mathrm{MSE}}

\newcommand{\COV}{\mathrm{COV}}
\newcommand{\aMV}{\overline{\mathrm{MV}}}
\newcommand{\aSDV}{\overline{\mathrm{SDV}}}
\newcommand{\aRDV}{\overline{\mathrm{RDV}}}
\newcommand{\aDPVAR}{\overline{\mathrm{DPVAR}}}
\newcommand{\aSDB}{\overline{\mathrm{SDB}}}
\newcommand{\aMB}{\overline{\mathrm{MB}}}
\newcommand{\BS}{\mathrm{BS}}

\newcommand{\pg}[1]{\textbf{#1.}}

\begin{document}



\twocolumn[

\aistatstitle{A Bias--Variance Decomposition for Ensembles over Multiple Synthetic Datasets}

\aistatsauthor{ Ossi Räisä \And Antti Honkela }

\aistatsaddress{ University of Helsinki \\ \texttt{ossi.raisa@helsinki.fi} \And University of Helsinki \\ \texttt{antti.honkela@helsinki.fi} } ]

\begin{abstract}
Recent studies have highlighted the benefits of generating multiple synthetic 
datasets for supervised learning, from increased accuracy to more effective model selection and 
uncertainty estimation. These benefits have clear empirical 
support, but the theoretical understanding of them is currently very light.
We seek to increase the theoretical understanding by deriving bias-variance 
decompositions for several settings of using multiple synthetic datasets, including 
differentially private synthetic data. Our theory 
yields a simple rule of thumb to select the appropriate 
number of synthetic datasets in the case of mean-squared error and Brier score.
We investigate how our theory works in practice with several real datasets,
downstream predictors and error metrics. As our theory predicts, multiple 
synthetic datasets often improve accuracy, while a single large synthetic dataset 
gives at best minimal improvement, showing that our insights are practically relevant.
\end{abstract}

\section{INTRODUCTION}\label{sec:introduction}
Synthetic data has recently attracted significant attention for several applications in machine learning.
The idea is to generate a dataset that preserves the population-level attributes of the 
real data. This makes the synthetic data useful for analysis, while also accomplishing a secondary task,
such as improving model evaluation~\citep{vanbreugelCanYouRely2023}, 
fairness~\citep{vanbreugelDECAFGeneratingFair2021}, data 
augmentation~\citep{antoniouDataAugmentationGenerative2018,dasConditionalSyntheticData2022} or
privacy~\citep{liewDataDistortionProbability1985,
rubin1993statistical}.
For privacy, \emph{differential privacy} (DP)~\citep{dworkCalibratingNoiseSensitivity2006}
is often combined with synthetic data 
generation~\citep{hardtSimplePracticalAlgorithm2012,mckennaWinningNISTContest2021} to achieve 
provable privacy protection, since 
releasing synthetic data without DP can be vulnerable to membership inference 
attacks~\citep{stadlerSyntheticDataAnonymisation2022,vanbreugelMembershipInferenceAttacks2023,meeusAchillesHeelsVulnerable2023}.

Several lines of work have considered generating multiple synthetic datasets from one real 
dataset for various purposes, including statistical 
inference~\citep{rubin1993statistical,raghunathanMultipleImputationStatistical2003,raisaConsistentBayesianInference2023}
and supervised learning~\citep{vanbreugelSyntheticDataReal2023}, the latter of which is our focus.

In supervised learning, \citet{vanbreugelSyntheticDataReal2023}
proposed an ensemble method of generative models.
They propose generating multiple synthetic datasets independently from 
an arbitrary generative model, training a predictive model separately on each 
synthetic dataset, and 
ensembling these models by averaging their predictions. They call this a
\emph{deep generative ensemble} (DGE). We drop the word ``deep'' in this paper and use 
the term \emph{generative ensemble},
as we do not require the generator to be deep in any sense.
The DGE was empirically demonstrated to be beneficial in several ways by 
\citet{vanbreugelSyntheticDataReal2023}, including 
predictive accuracy, model evaluation, model selection, and uncertainty estimation.
Followup work has applied DGEs to improve model evaluation under distribution shifts
and for small subgroups~\citep{vanbreugelCanYouRely2023}.

However, \citet{vanbreugelSyntheticDataReal2023} have very little theoretical analysis of 
the generative ensemble. Their theoretical justification assumes the data is generated 
from a posterior predictive distribution.
This assumption can be justified heuristically for
deep generative models through a Bayesian interpretation of deep
ensembles~\citep{
lakshminarayananSimpleScalablePredictive2017,
wilsonBayesianDeepLearning2020,
wilsonDeepEnsemblesApproximate2021}.
However, this justification only applies to generators with highly multi-modal losses, like neural networks.
It also does not provide any insight on how different choices in the setup, like the 
choice of downstream predictor, affect the performance of the ensemble.

The bias-variance decomposition~\citep{gemanNeuralNetworksBias1992} and its 
generalisations~\citep{uedaGeneralizationErrorEnsemble1996,woodUnifiedTheoryDiversity2023a}
are classical tools that provide insight into supervised learning problems.
The standard bias-variance decomposition from \citet{gemanNeuralNetworksBias1992}
considers predicting $y\in \R$ given features $x\in \calx$, using predictor 
$g(x; D)$ that receives training data $D$.
The mean-squared error (MSE) of $g$ can be decomposed into bias, variance,
and noise terms:
\begin{equation}
    \underbrace{\E_{D,y}[(y - g)^2]}_{\mathrm{MSE}} 
    = \underbrace{(f(x) - \E_D[g])^2}_{\mathrm{Bias}} 
    + \underbrace{\Var_D[g]}_{\mathrm{Variance}} 
    + \underbrace{\Var_{y}[y]}_{\mathrm{Noise}},
    \label{eq:classical-bias-variance-decomposition}
\end{equation}
where we have shortened
$g = g(x; D)$ and
$f(x) = \E_{y}[y]$ is the optimal predictor. $x$ is considered fixed, 
so all the random quantities in 
the decomposition are implicitly conditioned on $x$.
While MSE is typically only used with regression 
problems, the decomposition also applies to the Brier score~\citep{brier1950verification}
in classification, which is simply the MSE of class probabilities.

We seek to provide deeper theoretical understanding of generative ensembles through bias-variance 
decompositions, which provide a more fine-grained view of how different choices in the setup 
affect the end result.

\pg{Contributions}
\begin{enumerate}
    \item We derive a bias-variance decomposition for the MSE or 
    Brier score of 
    generative ensembles under an i.i.d.\ assumption in 
    Theorem~\ref{thm:mse-synthetic-data-decomposition}.
    This decomposition is simply a sum of interpretable terms, which makes it possible to predict 
    how various choices affect the error. 
    \item We derive several practical considerations from our decomposition, including 
    a simple rule of thumb
    to select the number of synthetic datasets in Section~\ref{sec:estimating-effect-multiple-syn-datasets}.
    In summary, 2 synthetic datasets give 50\% of the potential benefit from multiple synthetic 
    datasets, 10 give 90\% of the benefit and 100 give 99\% of the benefit.
    This also applies to bagging ensembles like random forests, which is likely to be of 
    independent interest.
    The benefit is a result of reduced variance, so the theory
    predicts high-variance models to receive the largest 
    benefit.
    \item We generalise the decomposition of Theorem~\ref{thm:mse-synthetic-data-decomposition}
    to differentially private (DP) generation algorithms that do not split their privacy budget 
    between the multiple synthetic datasets\footnote{
        Theorem~\ref{thm:mse-synthetic-data-decomposition} applies to DP algorithms that 
        split the privacy budget, but it is not clear if multiple synthetic datasets are 
        beneficial with these algorithms, as splitting the privacy budget between more synthetic
        datasets means that each one requires adding more noise, degrading the quality of 
        the synthetic data.
    } in Theorem~\ref{thm:mse-dp-synthetic-data-decomposition}, to non-i.i.d.\ synthetic 
    data in Appendix~\ref{sec:mse-decomposition-non-iid}, and
    to Bregman divergences in Appendix~\ref{app:bregman-divergence-decomposition}.
    \item We evaluate the performance of a generative ensemble on several 
    datasets, downstream prediction algorithms, and error metrics in Section~\ref{sec:experiments}. The 
    results show that out theory applies in practice:
    multiple synthetic datasets generally decrease all of the error metrics, 
    and our rule of thumb makes accurate predictions when the 
    error can be accurately estimated. In contrast, a single large synthetic 
    dataset provides small benefits at best, and can even increase error.
\end{enumerate}

\subsection{Related Work}
Ensembling generative models has been independently proposed several 
times~\citep{
wangEnsemblesGenerativeAdversarial2016,
choiWAICWhyGenerative2019,
luziEnsemblesGenerativeAdversarial2020,
chenSelfSupervisedBlindImage2023,
vanbreugelSyntheticDataReal2023}.
The inspiration of our work comes from \citet{vanbreugelSyntheticDataReal2023}, who proposed 
ensembling predictive models over multiple synthetic datasets, and empirically studied how 
this improves several aspects of performance in classification.

Generating multiple synthetic datasets has also been proposed with statistical inference in mind,
for both frequentist~\citep{
rubin1993statistical,
raghunathanMultipleImputationStatistical2003,
raisaNoiseawareStatisticalInference2023},
and recently Bayesian~\citep{raisaConsistentBayesianInference2023} inference.
These works use the multiple synthetic datasets to correct statistical inferences
for the extra uncertainty from synthetic data generation.

The bias-variance decomposition was originally derived by \citet{gemanNeuralNetworksBias1992}
for standard regressors using MSE as the loss. \citet{jamesVarianceBiasGeneral2003} generalised 
the decomposition to symmetric losses, and 
\citet{pfauGeneralizedBiasvarianceDecomposition2013} 
generalised it to Bregman divergences. 

\citet{uedaGeneralizationErrorEnsemble1996} were the first to 
study the MSE bias--variance decomposition with ensembles, and 
\citet{guptaEnsemblesClassifiersBiasVariance2022,woodUnifiedTheoryDiversity2023a}
later extended the ensemble decomposition to other losses. All of these also apply to generative 
ensembles, but they only provide limited insight for them, as they do not separate the 
synthetic data generation-related terms from the downstream-related terms.

\section{BIAS-VARIANCE DECOMPOSITIONS FOR GENERATIVE ENSEMBLES}\label{sec:mse-decomposition}
%\looseness=-1
In this section, we make our main theoretical contributions. We start by defining our 
setting in Section~\ref{sec:problem-setting}. We then derive a bias-variance decomposition for 
generative ensembles in Section~\ref{sec:mse-decomposition-ge}, and a simple 
rule of thumb that can be used to select the number of synthetic datasets in 
Section~\ref{sec:estimating-effect-multiple-syn-datasets}. The first 
decomposition does not apply to some differentially private synthetic data 
generation methods, so we generalise the decomposition to apply to those in 
Section~\ref{sec:mse-decomposition-dp}. We also present decompositions that 
apply non-i.i.d.\  settings and to Bregman divergences
in Appendices~\ref{sec:mse-decomposition-non-iid} and
\ref{app:bregman-divergence-decomposition},
but they are not as informative as the others.

\subsection{Problem Setting}\label{sec:problem-setting}
We consider multiple synthetic datasets $D_s^{1:m}$, each of which is used to train 
an instance of a predictive model $g$. These models are combined into an ensemble $\hat{g}$
by averaging, so 
\begin{equation}
    \hat{g}(x; D_s^{1:m}) = \frac{1}{m}\sum_{i=1}^m g(x; D_s^i).
\end{equation}
We allow $g$ to be random, so $g$ can for example internally select among several predictive 
models to use randomly.

We assume that each synthetic dataset $D_s^i$ is generated
from a generator with parameters $\theta_i$, and that
each generator is run with an independent random seed,
so generations are independent given the parameters
$\theta_{1:m}$. We also assume that each generator is 
trained on the real data $D_r$ with an independent 
random seed, and there are no other dependencies between
the training processes besides the real data, so
the $\theta_i$ are i.i.d. given $D_r$.
This is how synthetic datasets are sampled in DGE~\citep{vanbreugelSyntheticDataReal2023}, where $\theta_i$ 
are the parameters of the generative neural network from independent training runs. This also encompasses bootstrapping, where 
$\theta_i$ would be the real dataset\footnote{The $\theta_i$ random variables are i.i.d.\ if they 
are deterministically equal.}, and $p(D_s | \theta_i)$ is the bootstrapping.

We will also consider a more general setting that applies to some differentially private synthetic 
data generators that do not fit into this setting in Section~\ref{sec:mse-decomposition-dp}.
In Appendix~\ref{sec:mse-decomposition-non-iid}, we consider settings without the i.i.d.\ assumptions.

\subsection{Mean-squared Error Decomposition}\label{sec:mse-decomposition-ge}

Next, we present our main theorem.
With the conditional i.i.d. assumptions detailed in 
Section~\ref{sec:problem-setting}, all of the 
$\theta_i | D_r$ distributions are identical and 
independent, so we use $\theta$ as a shorthand for a random 
variable with this distribution in this section. 
Similarly, the $D_s^i | \theta_i$ distribution only 
depends on $\theta_i$, but not $i$ with these 
assumptions, so we use $D_s | \theta$ as a shorthand.

\begin{restatable}{theorem}{theoremmsesyntheticdatadecomposition}\label{thm:mse-synthetic-data-decomposition}
    Let the parameters for $m$ generators $\theta_i \sim p(\theta | D_r)$, $i=1, \dots, m$,
    be i.i.d. Let
    the synthetic datasets be $D_s^{i} \sim p(D_s | \theta_i)$ independently, and
    let $\hat{g}(x; D_s^{1:m}) = \frac{1}{m}\sum_{i=1}^m g(x; D_s^i)$. Then
    the mean-squared error in predicting $y$ from $x$ decomposes
    into six terms: model variance (MV), synthetic data variance (SDV), real data variance (RDV),
    synthetic data bias (SDB), model bias (MB), and noise $\Var_{y}[y]$:
    \begin{equation}
        \mathrm{MSE} = \frac{1}{m}\mathrm{MV} + \frac{1}{m}\mathrm{SDV} 
        + \mathrm{RDV}+ (\mathrm{SDB} + \mathrm{MB})^2 + \Var_y[y],
        \label{eq:mse-decomposition-mean-model}
    \end{equation}
    where
    \begin{equation}
        \begin{split}
            \mathrm{MSE} &= \E_{y, D_r, D_s^{1:m}}[(y - \hat{g}(x; D_s^{1:m}))^2] \\
            \mathrm{MV} &= \E_{D_r, \theta} \Var_{D_s|\theta}[g(x; D_s)] \\
            \mathrm{SDV} &= \E_{D_r}\Var_{\theta | D_r}\E_{D_s | \theta}[g(x; D_s)] \\
            \mathrm{RDV} &= \Var_{D_r}\E_{D_s|D_r}[g(x; D_s)] \\
            \mathrm{SDB} &= \E_{D_r}[f(x) - \E_{\theta | D_r}[f_\theta(x)]] \\
            \mathrm{MB} &= \E_{D_r, \theta}[f_\theta(x) - \E_{D_s|\theta}[g(x; D_s)]].
        \end{split}
    \end{equation}
    $f(x) = \E_{y}[y]$ is the optimal predictor for real data, 
    $\theta \sim p(\theta|D_r)$ is a single sample from the identical 
    distribution of the generator parameters $\theta_i$, 
    $D_s\sim p(D_s | \theta)$ is 
    a single sample of the synthetic data generating process given 
    $\theta$, and $f_\theta$ is the optimal predictor for the synthetic 
    data generating process with parameters $\theta$. 
    All random quantities are implicitly conditioned on $x$.
\end{restatable}
The proofs of all theorems are in Appendix~\ref{sec:missing-proofs}. 

To intuitively explain what each term measures, we 
split the whole generative ensemble pipeline into three 
steps:
\begin{enumerate}
    \item Sample real data $D_r$,
    \item given $D_r$, train generative model, resulting
    in $\theta$,
    \item given $\theta$, sample synthetic data $D_s$,
    train downstream model $g$ and make a prediction on $x$.
\end{enumerate}
In the full generative ensemble, steps 2 and 3 are repeated $m$ times.

MV measures the variance of step 3, averaged over 
the randomness in steps 1 and 2.
SDV measures the variance of the average prediction
from step 3 over the randomness in step 2, and also averages
over the randomness in step 1. RDV measures the variance
of the average prediction from the combination of steps 
2 and 3, over the randomness in step 1.

$f_\theta$ represents the optimal predictor if the 
synthetic data generation in step 3 were the real data
generating process. SDB measures how much the average $f(\theta)$, over the randomness in step 2,
differs from the actual optimal predictor $f$ on
the real data, and averages the difference over 
the randomness in step 1. MB measures how much 
the downstream model's average predictions, over the 
randomness in step 3, differ from $f_\theta$, and
averages the difference over the randomness in steps 1 and
2. $\Var_y[y]$ is the noise term also present in
the classical bias-variance decomposition in 
\eqref{eq:classical-bias-variance-decomposition}, which 
represents inherent randomness in the prediction problem
that cannot be removed by any predictor.

Note that the MV, SDV, RDV, SDB, MB and $\Var_y[y]$ terms 
do not depend
on $m$. Changing $m$ only affects the impact of MV and SDV
on MSE, as seen in \eqref{eq:mse-decomposition-mean-model}.

\pg{Multidimensional $y$}
For a multidimensional $y\in \R^d$, we have 
\begin{equation}
    \E_{y,D_r,D_s^{1:m}}\big[||y - \hat{g}||_2^2\big] = \sum_{j=1}^d \E_{y,D_r,D_s^{1:m}}\big[(y_j - \hat{g_j})^2\big],
\end{equation}
where we have shortened $\hat{g} = \hat{g}(x; D_s^{1:m})$ and
$y_j$ and $\hat{g}_j$ index over the dimensions. Since the right hand side is 
a sum of one-dimensional MSEs, our theory can also be applied to multi-dimensional $y$.

\pg{Relation to Brier Score}
The Brier score~\citep{brier1950verification} can be defined in two ways for binary 
classification. The first way to is pick a class, here class 0, and define
\begin{equation}
    \BS_1 = \E_{y, D}[(y_0 - g_0(x; D))^2],
\end{equation}
where $y_0$ is the indicator for the true class being 0, 
and $g_0(x; D)$ is the predicted probability of class 0. The second way is to define
\begin{equation}
    \BS_2 = \E_{y, D}[||y - g(x; D)||_2^2],
\end{equation}
where $y$ is a one-hot-encoding of the true class, and $g(x; D)$ is the vector of predicted 
probabilities. The second definition also extends to problems with more than two classes.
For binary classification, $\BS_2 = 2\cdot \BS_1$ since probabilities sum to 1,  
so it does not matter which class is picked for the first definition.

The first definition is an MSE, and the second is a sum of one-dimensional MSEs, so our 
theory applies to both, including multiclass problems through the extension to 
multidimensional $y$.

\paragraph{Practical Considerations}
We can derive the following practical considerations from 
Theorem~\ref{thm:mse-synthetic-data-decomposition}:
\begin{enumerate}
    \item There is a simple rule-of-thumb on how many synthetic datasets
    are beneficial: $m$ synthetic datasets give a $1 - \frac{1}{m}$ fraction
    of the possible benefit from multiple synthetic datasets.  For example,
    $m = 2$ gives 50\% of the benefit, $m = 10$ gives 90\% and $m = 100$ gives 
    99\%. This rule-of-thumb can be used to predict the MSE 
    with many synthetic datasets from the results on just two synthetic datasets.
    More details in Section~\ref{sec:estimating-effect-multiple-syn-datasets}.
    \item Increasing the number of synthetic datasets $m$ reduces the 
    impact of MV and SDV. This means that high-variance models, like 
    interpolating decision trees or 1-nearest neighbours, benefit the 
    most from multiple synthetic datasets. The size of this benefit 
    can be empirically estimated using the prediction rule of 
    Section~\ref{sec:estimating-effect-multiple-syn-datasets}.
    \item The quality metrics for a synthetic data generator should include metrics that 
    compare the distribution of the synthetic\footnote{Specifically, the conditional distribution of synthetic data given real data.} 
    data to the real distribution, instead of just 
    comparing a synthetic dataset to the real dataset. 
    %An example of the former is the downstream 
    %performance of predictive models, and 
    Examples of the latter are all metrics that return the optimal
    value for a generator that just returns the real data, like
    comparisons between the marginals of the real and synthetic 
    datasets, comparisons of downstream ML performance, and metrics like 
    density/coverage~\citep{naeemReliableFidelityDiversity2020}.
    One example of the former is the authenticity metric
    of~\citep{alaaHowFaithfulYour2022}, though it only 
    attempts to capture how well the generator generalises
    instead of copying the real data.
\end{enumerate}

The last point can be derived by considering two extreme scenarios:
\begin{enumerate}
    \item Generator is fitted perfectly, and generates from the real data generating distribution.
    Now
    \begin{equation*}
        \begin{split}
            \mathrm{MV} &= \Var_{D_r}[g(x; D_r)], \,\,\,\,
            \mathrm{MB} = \E_{D_r}[f(x) - g(x; D_r)] \\
            \mathrm{SDV} &= \mathrm{RDV} = \mathrm{SDB} = 0.
            %\phantom{x}
        \end{split}
    \end{equation*}
    With one synthetic dataset, the result is the standard bias-variance trade-off. With multiple
    synthetic datasets, the impact of MV can be reduced, reducing the error compared to just using 
    real data.
    \item Return the real data: $D_s | D_r = D_r$ deterministically.\footnote{
        Theorem~\ref{thm:mse-synthetic-data-decomposition} applies 
        in this scenario even with multiple synthetic datasets, as the deterministically identical 
        synthetic datasets are independent as random variables.
    }
    \begin{equation*}
        \begin{split}
            \mathrm{RDV} &= \Var_{D_r}[g(x; D_r)], \,\,\,
            \mathrm{MB} = \E_{D_r}[f(x) - g(x; D_r)] \\
            \mathrm{MV} &= \mathrm{SDV} = \mathrm{SDB} = 0.
            %\phantom{x}
        \end{split}
    \end{equation*}
    The result is the standard bias-variance trade-off. The number of synthetic datasets does not matter,
    as all of them would be the same anyway. 
\end{enumerate}
While both of these scenarios are unrealistic, they may be approximated by a well-performing algorithm. The generator from Scenario 2 is the optimal generator for metrics that compare the 
synthetic dataset to the real dataset, while the generator of Scenario 1 is optimal for metrics
that compare the synthetic data and real data distributions. Multiple synthetic datasets are 
only beneficial in Scenario 1, which means that metrics comparing the distributions are 
more meaningful when multiple synthetic datasets are considered.

\subsection{Estimating the effect of Multiple Synthetic Datasets}\label{sec:estimating-effect-multiple-syn-datasets}
Next, we consider estimating the variance terms MV and SDV in 
Theorem~\ref{thm:mse-synthetic-data-decomposition} from a small number of synthetic 
datasets and a test set. These estimates can then be used to asses if more synthetic 
datasets should be generated, and how many more are useful.

We can simplify Theorem~\ref{thm:mse-synthetic-data-decomposition} to
\begin{equation}
    \mathrm{MSE} = \frac{1}{m}\mathrm{MV} + \frac{1}{m}\mathrm{SDV} + \mathrm{Others},
    \label{eq:simple-mse-decomposition}
\end{equation}
where Others does not depend on the number of synthetic datasets $m$. 
The usefulness of more synthetic datasets clearly depends on the magnitude of 
$\mathrm{MV} + \mathrm{SDV}$ compared to Others.

Since MSE depends on $m$, we can add a subscript to denote the $m$ in question:
$\mathrm{MSE}_m$. 
Now \eqref{eq:simple-mse-decomposition} leads to the following corollary.
\begin{corollary}
    In the setting of Theorem~\ref{thm:mse-synthetic-data-decomposition}, we have
    \begin{equation}
        \mathrm{MSE}_{m} 
        = \mathrm{MSE}_{1} - \left(1 - \frac{1}{m}\right)(\mathrm{MV} + \mathrm{SDV}).
        \label{eq:mse-m-estimator}
    \end{equation}
\end{corollary}
\begin{proof}
    The claim follows by expanding $\MSE_1$ and $\MSE_m$ with \eqref{eq:simple-mse-decomposition}.
\end{proof}
If we have two synthetic datasets, we can estimate 
$\mathrm{MV} + \mathrm{SDV} = 2(\mathrm{MSE}_{1} - \mathrm{MSE}_{2})$, which gives the 
estimator
\begin{equation}
    \mathrm{MSE}_{m} 
    = \mathrm{MSE}_{1} - 2\left(1 - \frac{1}{m}\right)(\mathrm{MSE}_{1} - \mathrm{MSE}_{2}).
    \label{eq:mse-m-estimator-two-syn-datasets}
\end{equation}

If we have more than two synthetic datasets, we can set $x_m = 1 - \frac{1}{m}$ and 
$y_m = \mathrm{MSE}_m$ in \eqref{eq:mse-m-estimator}:
\begin{equation}
    y_m = y_1 + x_m(\mathrm{MV} + \mathrm{SDV}),
    \label{eq:mv+sdv-estimator-linear-regression}
\end{equation}
so we can estimate $\mathrm{MV} + \mathrm{SDV}$ from linear regression on $(x_m, y_m)$.
However, this will likely have a limited effect on the accuracy of the $\mathrm{MSE}_m$ 
estimates, as it will not
reduce the noise in estimating $\mathrm{MSE}_{1}$, which has a significant effect in
\eqref{eq:mse-m-estimator}.

From \eqref{eq:mse-m-estimator}, we see that $\mathrm{MV} + \mathrm{SDV}$ is the 
maximum reduction in MSE that can be obtained from multiple synthetic datasets.
This means that $2(\mathrm{MSE}_1 - \mathrm{MSE}_2)$, or the linear regression estimate
from \eqref{eq:mv+sdv-estimator-linear-regression},
can be used as a diagnostic to check whether generating multiple synthetic datasets 
is worthwhile.

All terms in \eqref{eq:simple-mse-decomposition}-\eqref{eq:mv+sdv-estimator-linear-regression}
depend on the target features $x$. We would like our estimates to be useful 
for typical $x$, so we will actually want to estimate 
$\E_x(\mathrm{MSE}_m)$. Equations 
\eqref{eq:simple-mse-decomposition}-\eqref{eq:mv+sdv-estimator-linear-regression}
remain valid if we take the expectation over $x$, so we can simply replace 
the MSE terms with their estimates that are computed from a test set. 

Computing the estimates in practice will require that the privacy risk of publishing the 
test MSE is considered acceptable. The MSE for the estimate can also be computed 
from a separate validation set to avoid overfitting to the test set, but the risk 
of overfitting is small in this case, as $m$ has a monotonic effect on the MSE.
Both of these caveats can be avoided by choosing $m$ using the rule of thumb that 
$m$ synthetic datasets give a $1 - \frac{1}{m}$ of the potential benefit of multiple 
synthetic datasets, which is a consequence of \eqref{eq:mse-m-estimator}.

Note that this MSE estimator can be applied to bagging ensembles~\citep{breimanBaggingPredictors1996} 
like random forests~\citep{breimanRandomForests2001}, since bootstrapping is a very simple form 
of synthetic data generation.

\subsection{Differentially Private Synthetic Data Generators}\label{sec:mse-decomposition-dp}
Generating and releasing multiple synthetic datasets could increase the associated 
disclosure risk. One solution to this is \emph{differential privacy} 
(DP)~\citep{dworkCalibratingNoiseSensitivity2006,dworkAlgorithmicFoundationsDifferential2014},
which is a property of an algorithm that formally bounds the privacy leakage that can 
result from releasing the output of that algorithm. DP gives a quantitative upper 
bound on the privacy leakage, which is known as the privacy budget.
Achieving DP requires adding extra noise to some point in the algorithm, lowering the 
utility of the result.

If the synthetic data is to be generated with DP, there are two possible ways 
to handle the required noise addition. The first is splitting the privacy budget across the $m$
synthetic datasets, and run the DP generation algorithm separately $m$ times. 
Theorem~\ref{thm:mse-synthetic-data-decomposition} applies in this setting. However, it is not 
clear if multiple synthetic datasets are beneficial in this case, as splitting the privacy 
budget requires adding more noise to each synthetic dataset. This also means that the rule of 
thumb from Section~\ref{sec:estimating-effect-multiple-syn-datasets} will not apply.
Most DP synthetic data generation algorithms would fall into this
category~\citep{
aydoreDifferentiallyPrivateQuery2021,
chenGSWGANGradientsanitizedApproach2020,
harderDPMERFDifferentiallyPrivate2021,
hardtSimplePracticalAlgorithm2012,
liuIterativeMethodsPrivate2021,
mckennaGraphicalmodelBasedEstimation2019,
mckennaWinningNISTContest2021} if used to generate multiple synthetic datasets.

The second possibility is generating all synthetic datasets based on a single application of a 
DP mechanism. Specifically, a noisy summary $\sdp$ of the real data is released under DP.
The parameters $\theta_{1:m}$ are then sampled i.i.d.\ conditional on $\sdp$, and the 
synthetic datasets are sampled conditionally on the $\theta_{1:m}$. This setting includes 
algorithms that release a posterior distribution under DP, and use the posterior to 
generate synthetic data, like the NAPSU-MQ algorithm~\citep{raisaNoiseawareStatisticalInference2023} and 
DP variational inference 
(DPVI)~\citep{jalkoDifferentiallyPrivateVariational2017,jalkoPrivacypreservingDataSharing2021}.\footnote{
In DPVI, $\sdp$ would be the variational approximation to the posterior.
}

The synthetic datasets are not i.i.d.\ given the real data in the second setting, so the setting
described in Section~\ref{sec:problem-setting} and assumed in 
Theorem~\ref{thm:mse-synthetic-data-decomposition} does not apply. However, the synthetic 
datasets are i.i.d.\ given the noisy summary $\sdp$, so we obtain a similar decomposition as before.
\begin{restatable}{theorem}{theoremmsedpsyntheticdatadecomposition}\label{thm:mse-dp-synthetic-data-decomposition}
    Let the parameters for $m$ generators $\theta_i \sim p(\theta | \sdp)$, $i=1,\dotsc,m$,
    be i.i.d.\ given a DP summary $\sdp$. Let
    the synthetic datasets be $D_s^{i} \sim p(D_s | \theta_i)$ independently, and
    let $\hat{g}(x; D_s^{1:m}) = \frac{1}{m}\sum_{i=1}^m g(x; D_s^i)$. Then
    \begin{equation}
        \begin{split}
        \mathrm{MSE} &= \frac{1}{m}\mathrm{MV} + \frac{1}{m}\mathrm{SDV} 
        + \mathrm{RDV} + \mathrm{DPVAR} 
        \\&+ (\mathrm{SDB} + \mathrm{MB})^2 + \Var_{y}[y],
        \end{split}
        \label{eq:mse-decomposition-latent-var-mean-model}
    \end{equation}
    where
    \begin{equation}
        \begin{split}
            \mathrm{MSE} &= \E_{y, D_r, \sdp, D_s^{1:m}}[(y - \hat{g}(x; D_s^{1:m}))^2] \\
            \mathrm{MV} &= \E_{D_r, \sdp, \theta} \Var_{D_s|\theta}[g(x; D_s)] \\
            \mathrm{SDV} &= \E_{D_r, \sdp}\Var_{\theta | \sdp}\E_{D_s | \theta}[g(x; D_s)] \\
            \mathrm{RDV} &= \Var_{D_r}\E_{D_s|D_r}[g(x; D_s)]
        \end{split}
    \end{equation}
    \begin{equation}
        \begin{split}
            \mathrm{DPVAR} &= \E_{D_r}\Var_{\sdp|D_r}\E_{D_s|\sdp}[g(x; D_s)] \\
            \mathrm{SDB} &= \E_{D_r, \sdp}[f(x) - \E_{\theta | \sdp}[f_\theta(x)]] \\
            \mathrm{MB} &= \E_{D_r, \sdp, \theta}[f_\theta(x) - \E_{D_s|\theta}[g(x; D_s)]].
        \end{split}
    \end{equation}
    $f(x) = \E_{y}[y]$ is the optimal predictor for real data, 
    $\theta \sim p(\theta|D_r)$ is a single sample from the distribution of the 
    generator parameters, $D_s\sim p(D_s | \theta)$ is 
    a single sample of the synthetic data generating process, and
    $f_\theta$ is the optimal predictor for the synthetic data generating process
    with parameters $\theta$. All random quantities are implicitly conditioned on $x$.
\end{restatable}

The takeaways from Theorem~\ref{thm:mse-dp-synthetic-data-decomposition} are mostly the same 
as from Theorem~\ref{thm:mse-synthetic-data-decomposition}, and the estimator from 
Section~\ref{sec:estimating-effect-multiple-syn-datasets} also applies. The main difference is the 
DPVAR term in Theorem~\ref{thm:mse-dp-synthetic-data-decomposition}, which accounts for the 
added DP noise. As expected, the impact of DPVAR cannot be reduced with additional synthetic 
datasets.

\section{EXPERIMENTS}\label{sec:experiments}

In this section, we describe our experiments. The common theme in all 
of them is generating synthetic data and evaluating the performance of several downstream 
prediction algorithms trained on the synthetic data. The performance evaluation uses a 
test set of real data, which is split from the whole dataset before generating synthetic data. Our code is available at \url{https://github.com/DPBayes/generative-ensemble-bias-variance-decomposition}.

The downstream algorithms we consider are nearest neighbours with 1 or 5 neighbours (1-NN and 5-NN),
decision tree (DT), random forest (RF), gradient boosted trees (GB), a multilayer perceptron (MLP) and a 
support vector machine (SVM) for both classification and regression. We also 
use linear regression (LR) and ridge regression (RR) on regression tasks, and logistic regression (LogR) on 
classification tasks, though we omit linear regression from the main text plots, as its results
are nearly identical to ridge regression. Decision trees and 1-NN have a very high variance, 
as both interpolate 
the training data with the hyperparameters we use. Linear, ridge, and logistic regression 
have fairly small variance in contrast. Appendix~\ref{sec:experiment-details} contains more 
details on the experimental setup, including details on the datasets, and downstream 
algorithm hyperparameters.

In each of our figures comparing prediction performance, 
we have included a horizontal 
black line showing the performance of the best downstream
predictor without synthetic data. The downstream models
include a random forest and gradient boosted trees, which 
are ensemble methods, so this line serves as a baseline
for ensembles without synthetic data. Another relevant
baseline is the performance of an ensemble with only one
synthetic dataset, which is given by the results of 
random forests and gradient boosted trees with $m = 1$.

\subsection{Effect of Multiple Synthetic Datasets}\label{sec:main-experiment}

\begin{figure*}[t]
    \begin{subfigure}{\textwidth}
        \centering
        \includegraphics[width=\textwidth]{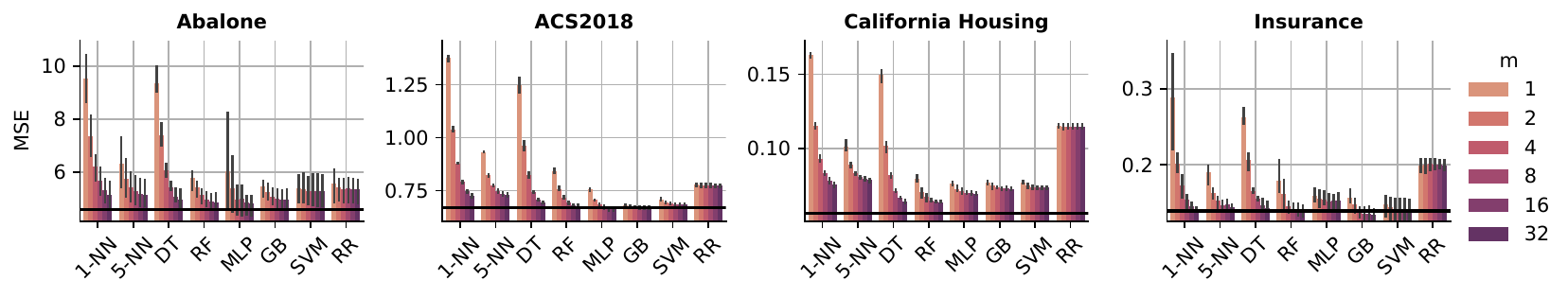}

        \vspace{-2mm}
        \caption{Regression datasets}
        \label{fig:main-results-regression}
    \end{subfigure}
    \begin{subfigure}{\textwidth}
        \centering
        \includegraphics[width=0.75\textwidth]{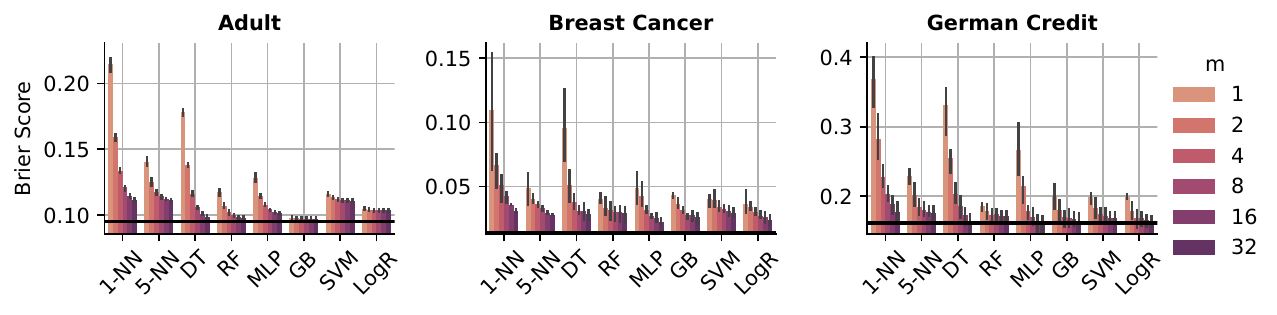}
        
        \vspace{-2mm}
        \caption{Classification datasets}
        \label{fig:main-results-classification}
    \end{subfigure}
    \caption{
        MSE on regression datasets (a) or Brier score on classification datasets (b) of the 
        ensemble of downstream predictors, with varying 
        number of synthetic datasets $m$ from synthpop. Increasing the number of 
        synthetic datasets generally decreases both metrics, especially for decision trees and 1-NN.
        The predictors are nearest neighbours 
        with 1 or 5 neighbours (1-NN and 5-NN), decision tree (DT), random forest (RF), a multilayer perceptron (MLP), 
        gradient boosted trees (GB), a support vector machine (SVM), ridge regression (RR) and 
        logistic regression (LogR). 
        The black line is the MSE of the best predictor on real data. 
        Tables~\ref{table:abalone-results} to 
        \ref{table:german-credit-brier-results}
        in the Appendix contain 
        the numbers from the plots.
    }
    \label{fig:main-results}
    \vspace{-4mm}
\end{figure*}

As our first experiment, we evaluate the performance of the synthetic data ensemble on 7 datasets,
containing 4 regression and 3 classification tasks. See Appendix~\ref{sec:dataset-details} for 
details on the datasets. We use the synthetic data generators 
DDPM~\citep{kotelnikovTabDDPMModellingTabular2023} and 
synthpop~\citep{nowokSynthpopBespokeCreation2016}, which we selected after a preliminary experiment 
described in Appendix~\ref{sec:synthetic-data-algorithms-comparison}. We only plot
the results from synthpop in the main text to save space, and defer the results of DDPM to 
Appendix~\ref{sec:extra-plots}.
We generate 32 synthetic datasets, of which between 1 and 32 are used to train the ensemble.
The results are averaged over 3 runs with different train-test splits. We compute error bars as 95\% confidence intervals
obtained from bootstrapping over the repeats.

On the regression datasets, our error metric is MSE, which is the subject of 
Theorem~\ref{thm:mse-synthetic-data-decomposition}. The results in 
Figure~\ref{fig:main-results-regression} show that 
a larger number of synthetic datasets generally decreases MSE. The decrease is especially clear 
with downstream algorithms that have a high variance like decision trees and 
1-NN. Low-variance algorithms like ridge regression have very little if any 
decrease from multiple synthetic datasets. This is consistent with 
Theorem~\ref{thm:mse-synthetic-data-decomposition}, where the number of synthetic datasets 
only affects the variance-related terms.

On the classification datasets, we consider 4 error metrics. Brier 
score~\citep{brier1950verification} is MSE of the class probability 
predictions, so Theorem~\ref{thm:mse-synthetic-data-decomposition} applies to it. Cross entropy is 
a Bregman divergence, so Theorem~\ref{thm:bregman-synthetic-data-decomposition} from 
Appendix~\ref{app:bregman-divergence-decomposition} applies to it.
We also included accuracy and area under the ROC curve (AUC) even though our theory does not 
apply to them, as they are common and interpretable error metrics, so it is interesting
to see how multiple synthetic datasets affect them. We use their complements in the plots, so 
that lower is better for all plotted metrics. We only present the Brier score results 
in the main text in Figure~\ref{fig:main-results-classification}, and defer the rest
to Figures~\ref{fig:adult-results}, \ref{fig:breast-cancer-results} and 
\ref{fig:german-credit-results} in Appendix~\ref{sec:extra-plots}.

Because Theorem~\ref{thm:bregman-synthetic-data-decomposition} only applies to cross entropy when
averaging log probabilities instead of probabilities, we compare both ways of averaging in the 
Appendix~\ref{sec:extra-plots}, but only include probability averaging in the main text.

The results on the classification datasets in Figure~\ref{fig:main-results-classification}
are similar to the regression experiment. A larger 
number of synthetic datasets generally decreases the score, especially for the 
high-variance models. 

In Figure~\ref{fig:variance-estimation-subset},
we estimate the MV and SDV terms of
the decomposition in Theorem~\ref{thm:mse-synthetic-data-decomposition}.
The results show that MV depends mostly on the 
downstream predictor, while SDV also depends 
on the synthetic data generator.
The results also confirm 
our claims on the model variances: decision trees and 1-NN have a high variance, while 
linear, ridge and logistic regression have a low variance. 
See Appendix~\ref{app:variance-estimation-experiment},
for details, and Figure~\ref{fig:variance-estimation}
for results on all datasets.

\begin{figure*}
    \begin{subfigure}{0.5\textwidth}
        \centering 
        \includegraphics[width=\textwidth]{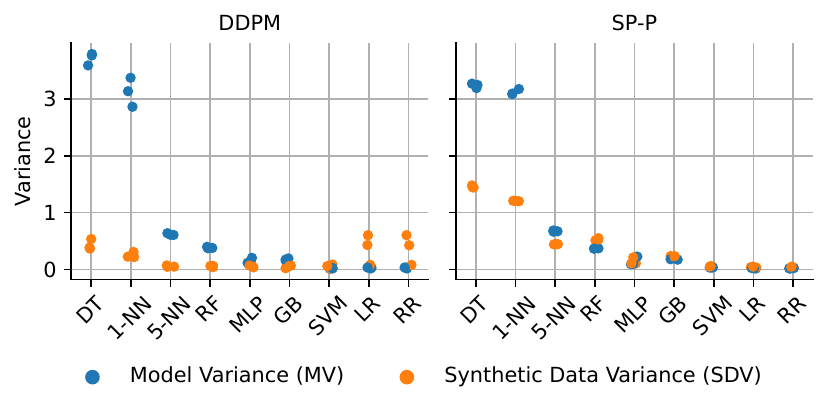}
        \vspace{-6mm}
        \caption{Abalone}
    \end{subfigure}
    \begin{subfigure}{0.5\textwidth}
        \centering 
        \includegraphics[width=\textwidth]{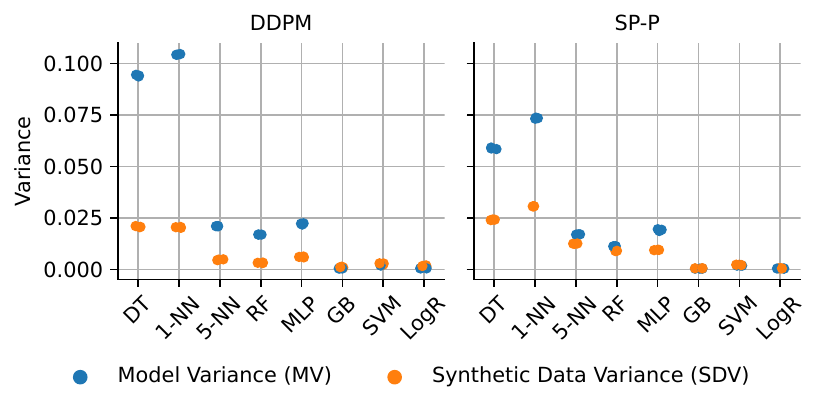}
        \vspace{-6mm}
        \caption{Adult}
    \end{subfigure}
    \caption{
        Estimating the MV and SDV terms from the decomposition. Decision trees 
        have high variances on all datasets, while linear, ridge and logistic regression
        have low variances. MV depends mostly on the predictor, while SDV depends 
        on both the predictor and synthetic data generation algorithm. The 
        points are the averages of estimated MV and SDV,
        averaged over the test data, from 3 repeats with different train-test splits.
        See Figure~\ref{fig:variance-estimation} in
        the Appendix for results on all datasets.
    }
    \label{fig:variance-estimation-subset}
\end{figure*}

In Appendix~\ref{app:one-large-synthetic-dataset} we compare an alternative
to the ensemble of multiple synthetic datasets: generating a single 
large synthetic datasets with an equal number of datapoints as all 
the multiple synthetic datasets combined. One could expect generating a
larger synthetic dataset to also reduce variances while saving on
the computational cost of training multiple generative models.
However, in the cases we examined in Figure~\ref{fig:one-large-results-regression}, a single larger dataset 
gave at best a small improvement, and sometimes even increased
the error. In contrast, adding more synthetic datasets often 
decreases error and never increases it.
As a result, as a default choice, we recommend generating as many synthetic 
datasets as possible that have the same size as the 
original. We do not recommend making 
the synthetic datasets smaller, since the computational
saving is likely small, and the additional generators
that could be trained with the savings have diminishing 
returns due to the rule-of-thumb that $m$ synthetic datasets
give a $(1 - \frac{1}{m})$ fraction of the benefit.

\begin{figure*}[t]
        \centering
        \includegraphics[width=\textwidth]{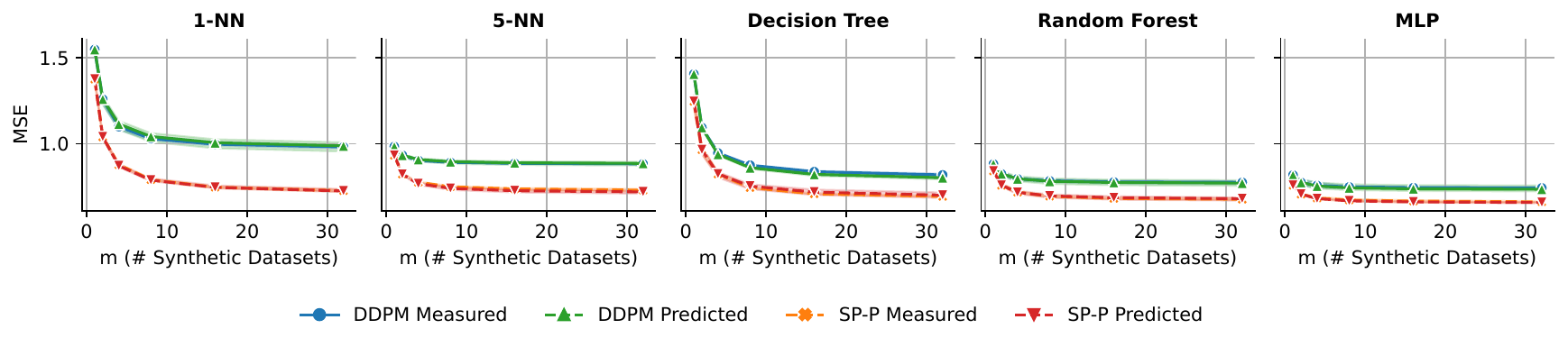}
        \vspace{-3mm}
    \caption{
        MSE or Brier score prediction on the ACS 2018 dataset. The predictions are very accurate on this 
        dataset. The solid lines for DDPM and synthpop (SP-P) show the same error MSE or Brier score
        as Figure~\ref{fig:main-results}, 
        while the dashed lines show predicted MSE or Brier score. 
        1-NN and 5-NN are nearest neighbours with 1 or 5 neighbours.
        We omitted downstream algorithms with uninteresting flat curves. See 
        Figure~\ref{fig:mse-prediction-regression1} in the Appendix for the full
        figure, and Figures~\ref{fig:mse-prediction-regression2} and 
        \ref{fig:mse-prediction-classification1} for the other datasets. 
        Tables~\ref{table:mse-prediction-acs-2018} and \ref{table:mse-prediction-german-credit} 
        in the Appendix contain the 
        numbers from the plots. 
    }
    \label{fig:mse-prediction-main}
    \vspace{-4mm}
\end{figure*}

\subsection{Predicting Performance from Two Synthetic Datasets}\label{sec:mse-prediction-experiment}
Next, we evaluate the predictions our rule of thumb from 
Section~\ref{sec:estimating-effect-multiple-syn-datasets} makes. To recap, our rule of 
thumb predicts that the maximal benefit from multiple synthetic datasets is 
$2(\mathrm{MSE}_1 - \mathrm{MSE}_2)$, and $m$ synthetic datasets achieve 
a $1 - \frac{1}{m}$ fraction of this benefit, as shown in 
\eqref{eq:mse-m-estimator-two-syn-datasets}.

To evaluate the predictions from the rule, we estimate the MSE on regression tasks 
and Brier score on classification tasks for one and two synthetic datasets from the
test set. The setup is otherwise identical to Section~\ref{sec:main-experiment},
and the train-test splits are the same. 
We plot the predictions from the rule, and compare them with 
the measured test errors with more than two synthetic datasets. 

Figure~\ref{fig:mse-prediction-main} contains the results for the ACS 2018 datasets,
and Figures~\ref{fig:mse-prediction-regression1} to \ref{fig:mse-prediction-classification1}
in the Appendix contain the results for the other datasets. The predictions are very accurate 
on ACS 2018, and reasonably accurate on the other datasets. 
The variance of the prediction depends heavily on the variance of the errors computed 
from the test data.

We also evaluated the rule on random forests without synthetic data, as it also applies to 
them. In this setting, the number of trees in the random forest is analogous to the 
number of synthetic datasets. We use the same datasets as in the previous experiments and 
use the same train-test splits of the real data. The results are in 
Figure~\ref{fig:random-forest-mse-prediction} in the Appendix. The prediction is accurate 
when the test error is accurate, but can have high variance.

\subsection{Differentially Private Synthetic Data}

\begin{figure*}
    \centering 
    \includegraphics[width=0.9\textwidth]{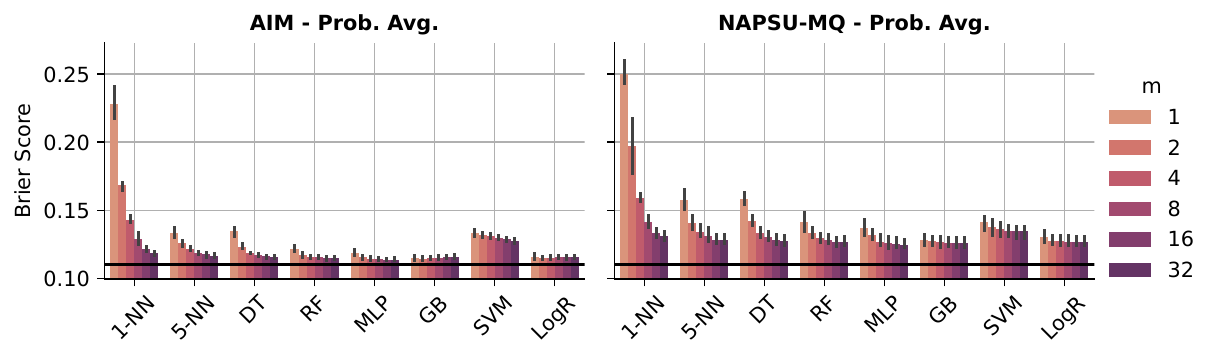}
    \caption{Brier score of the ensemble of downstream predictors with varying numbers of synthetic datasets 
    $m$, generated with the DP methods AIM or NAPSU-MQ from the Adult dataset with a reduced set of 
    features. Increasing the number of datasets generally decreases the score, even with AIM, which 
    splits the privacy budget between $m$ synthetic datasets. 
    The privacy parameters are $\epsilon = 1.5$, $\delta = n^{-2} \adultdelta$.
    The predictors are 
    the same as in Figure~\ref{fig:main-results}.
    The black lines show the loss of the best non-DP
    downstream predictor trained on real data. 
    Table~\ref{table:dp-experiment-brier} contains the numbers from the 
    plots, and Figure~\ref{fig:dp-experiment-all-metrics} contains plots of the other error metrics.
    }
    \label{fig:dp-experiment-brier}
    \vspace{-4mm}
\end{figure*}

In this experiment, we evaluate the performance of the generative ensemble in the setting of
Theorem~\ref{thm:mse-dp-synthetic-data-decomposition}, where $m$ synthetic datasets are generated from a 
single noisy summary $\sdp$ of the real data. We compare with splitting the privacy 
budget between $m$ synthetic datasets.

The algorithms we use are AIM~\citep{mckennaAIMAdaptiveIterative2022}, which needs to split the privacy 
budget between the $m$ synthetic datasets, and NAPSU-MQ~\citep{raisaNoiseawareStatisticalInference2023},
which uses a single noisy summary. The dataset is the Adult dataset with a reduced set of columns,
which is needed to keep the runtime of NAPSU-MQ reasonable. The downstream task is classification, so we 
use the same 4 metrics, and both probability and log-probability averaging, as in the non-DP classification 
experiment. We use fairly strict privacy parameters of $\epsilon = 1.5$, $\delta = n^{-2} \adultdelta$.
See Appendix~\ref{sec:experiment-details}
for the full details of the setup.

The results for Brier score are in Figure~\ref{fig:dp-experiment-brier}, and the results for all metrics 
are in Figure~\ref{fig:dp-experiment-all-metrics} in the Appendix. Increasing $m$ often always 
improves the results, which is expected for NAPSU-MQ due to Theorem~\ref{thm:mse-dp-synthetic-data-decomposition},
but somewhat surprising for AIM, which splits the privacy budget. For AIM, $m = 4$ looks like a sensible default choice
for the well-performing downstream predictors, since 
performance does not greatly change either way 
with $m > 4$.
Once again, the high-variance predictors 
1-NN, 5-NN and decision trees benefit the most from multiple synthetic datasets.
AIM always outperforms NAPSU-MQ, which is 
likely a result of AIM being able to handle a more complicated set of input queries than NAPSU-MQ.

\section{DISCUSSION}\label{sec:discussion}

\pg{Limitations}
Our main theory assumes that the synthetic datasets are generated i.i.d.\ given either 
the real data or a DP summary, which for example leaves out
generative ensembles that explicitly encourage diversity
between synthetic datasets in some way.
We generalise the bias-variance decompositions to 
non-i.i.d.\ settings in Appendix~\ref{sec:mse-decomposition-non-iid}. The implications of 
the decompositions, like the rule-of-thumb on the number of synthetic datasets, 
do not apply in general if i.i.d.\ assumption is removed, 
but they can still apply with additional assumptions. 
We give an example in Appendix~\ref{sec:mse-decomposition-non-iid}.

\pg{Conclusion}
We derived bias-variance decompositions for using synthetic data in several cases:
for MSE or Brier score with i.i.d.\ synthetic datasets given the real data
(Section~\ref{sec:mse-decomposition-ge})
and MSE with i.i.d.\ synthetic datasets given a DP summary of the real data 
(Section~\ref{sec:mse-decomposition-dp}).
We generalised these decompositions to non-i.i.d.\ synthetic datasets
(Appendix~\ref{sec:mse-decomposition-non-iid})
and Bregman divergences (Appendix~\ref{app:bregman-divergence-decomposition}).
The decompositions make actionable predictions, such as yielding a simple rule of 
thumb that can be used to select the number of synthetic datasets.
We empirically examined the performance 
of generative ensembles on several real datasets and downstream predictors,
and found that the predictions of the theory generally hold in practice 
(Section~\ref{sec:experiments}).
These findings significantly increase the theoretical understanding of 
generative ensembles, which is very limited in prior literature.

\subsubsection*{Acknowledgments}
This work was supported by the Research Council of Finland 
(Flagship programme: Finnish Center for Artificial Intelligence, 
FCAI as well as Grants 356499 and 359111), the Strategic Research Council 
at the Research Council of Finland (Grant 358247)
as well as the European Union (Project
101070617). Views and opinions expressed are however
those of the author(s) only and do not necessarily reflect
those of the European Union or the European Commission. Neither the European 
Union nor the granting authority can be held responsible for them.
The authors wish to thank the Finnish Computing Competence 
Infrastructure (FCCI) for supporting this project with 
computational and data storage resources.

\bibliography{Multiple_Synthetic_Datasets_ML}
\bibliographystyle{icml2024}

%%%%%%%%%%%%%%%%%%%%%%%%%%%%%%%%%%%%%%%%%%%%%%%%%%%%%%%%%%%%
\section*{Checklist}

 \begin{enumerate}

 \item For all models and algorithms presented, check if you include:
 \begin{enumerate}
   \item A clear description of the mathematical setting, assumptions, algorithm, and/or model. [Yes]
   \item An analysis of the properties and complexity (time, space, sample size) of any algorithm. [Not Applicable]
   \item (Optional) Anonymized source code, with specification of all dependencies, including external libraries. [Yes]
 \end{enumerate}

 \item For any theoretical claim, check if you include:
 \begin{enumerate}
   \item Statements of the full set of assumptions of all theoretical results. [Yes]
   \item Complete proofs of all theoretical results. [Yes]
   \item Clear explanations of any assumptions. [Yes]     
 \end{enumerate}

 \item For all figures and tables that present empirical results, check if you include:
 \begin{enumerate}
   \item The code, data, and instructions needed to reproduce the main experimental results (either in the supplemental material or as a URL). [Yes]
   \item All the training details (e.g., data splits, hyperparameters, how they were chosen). [Yes]
         \item A clear definition of the specific measure or statistics and error bars (e.g., with respect to the random seed after running experiments multiple times). [Yes]
         \item A description of the computing infrastructure used. (e.g., type of GPUs, internal cluster, or cloud provider). [Yes]
 \end{enumerate}

 \item If you are using existing assets (e.g., code, data, models) or curating/releasing new assets, check if you include:
 \begin{enumerate}
   \item Citations of the creator If your work uses existing assets. [Yes]
   \item The license information of the assets, if applicable. [Yes]
   \item New assets either in the supplemental material or as a URL, if applicable. [Yes]
   \item Information about consent from data providers/curators. [No] None of the providers of the datasets we used 
   provide information about consent.
   \item Discussion of sensible content if applicable, e.g., personally identifiable information or offensive content. [Not Applicable]
 \end{enumerate}

 \item If you used crowdsourcing or conducted research with human subjects, check if you include:
 \begin{enumerate}
   \item The full text of instructions given to participants and screenshots. [Not Applicable]
   \item Descriptions of potential participant risks, with links to Institutional Review Board (IRB) approvals if applicable. [Not Applicable]
   \item The estimated hourly wage paid to participants and the total amount spent on participant compensation. [Not Applicable]
 \end{enumerate}

 \end{enumerate}

%%%%%%%%%%%%%%%%%%%%%%%%%%%%%%%%%%%%%%%%%%%%%%%%%%%%%%%%%%%%%%%%%%%%%%%%%%%%%%%
%%%%%%%%%%%%%%%%%%%%%%%%%%%%%%%%%%%%%%%%%%%%%%%%%%%%%%%%%%%%%%%%%%%%%%%%%%%%%%%
% APPENDIX
%%%%%%%%%%%%%%%%%%%%%%%%%%%%%%%%%%%%%%%%%%%%%%%%%%%%%%%%%%%%%%%%%%%%%%%%%%%%%%%
%%%%%%%%%%%%%%%%%%%%%%%%%%%%%%%%%%%%%%%%%%%%%%%%%%%%%%%%%%%%%%%%%%%%%%%%%%%%%%%
\newpage

\appendix
\onecolumn

\renewcommand\thefigure{S\arabic{figure}}
\renewcommand\thetable{S\arabic{table}}
\setcounter{figure}{0}

\section{MISSING PROOFS}\label{sec:missing-proofs}

\theoremmsesyntheticdatadecomposition*
\begin{proof}
    With $m$ synthetic datasets $D_s^{1:m}$ and model $\hat{g}(x, D_s^{1:m})$ that combines the synthetic 
    datasets, the classical bias-variance decomposition gives
    \begin{align}
        \E_{y, D_r, D_s^{1:m}}[(y - \hat{g}(x; D_s^{1:m}))^2]
        = (f(x) - \E_{D_r, D_s^{1:m}}[\hat{g}(x; D_s^{1:m})])^2 
        + \Var_{D_r, D_s^{1:m}}[\hat{g}(x; D_s^{1:m})] + \Var_{y}[y].
    \end{align}

    Using the independence of the synthetic datasets, these can be decomposed further:
    \begin{equation}
        \E_{D_r, D_s^{1:m}}[\hat{g}(x; D_s^{1:m})] 
        = \E_{D_r}\E_{D_s^{1:m}|D_r}\left[\frac{1}{m}\sum_{i=1}^m g(x, D_s^i)\right]
        = \E_{D_r}\E_{D_s|D_r}[g(x, D_s)]
        = \E_{D_r, D_s}[g(x, D_s)],
    \end{equation}
    
    \begin{equation}
        \begin{split}
            \Var_{D_r, D_s^{1:m}}[\hat{g}(x; D_s^{1:m})]
            &= \E_{D_r}\Var_{D_s^{1:m}|D_r}\left[\frac{1}{m} \sum_{i=1}^m g(x; D_s^i)\right]
            + \Var_{D_r}\E_{D_s^{1:m}|D_r}\left[\frac{1}{m} \sum_{i=1}^m g(x; D_s^i)\right]
            \\&= \frac{1}{m^2} \E_{D_r}\Var_{D_s^{1:m}|D_r}\left[\sum_{i=1}^m g(x; D_s^i)\right]
            + \Var_{D_r}\E_{D_s|D_r}\left[g(x; D_s)\right]
            \\&= \frac{1}{m} \E_{D_r}\Var_{D_s|D_r}\left[g(x; D_s)\right]
            + \Var_{D_r}\E_{D_s|D_r}\left[g(x; D_s)\right],
            \label{eq:variance-decomposition-mean-model}
        \end{split}
    \end{equation}
    and
    \begin{equation}
        \Var_{D_s|D_r}[g(x; D_s)] = \E_{\theta|D_r}\Var_{D_s|\theta}[g(x; D_s)] + \Var_{\theta|D_r}\E_{D_s|\theta}[g(x, D_s)].
    \end{equation}
    
    The bias can be decomposed with $f_\theta(x)$:
    \begin{equation}
        \begin{split}
            f(x) - \E_{D_r, D_s}[g(x; D_s)] 
            &= \E_{D_r, \theta}[f(x) - f_\theta(x) + f_\theta(x) - \E_{D_s|\theta}[g(x; D_s)]]
            \\&= \E_{D_r}[f(x) - \E_{\theta|D_r}[f_\theta(x)]] 
            + \E_{D_r, \theta}[f_\theta(x) - \E_{D_s|\theta}[g(x; D_s)]]
        \end{split}
    \end{equation}
    
    Combining all of these gives the claim.
\end{proof}

\theoremmsedpsyntheticdatadecomposition*
\begin{proof}
    Using $\sdp$ in place of the real data in Theorem~\ref{thm:mse-synthetic-data-decomposition}
    gives
    \begin{equation}
        \mathrm{MSE} = \frac{1}{m}\mathrm{MV} + \frac{1}{m}\mathrm{SDV} 
        + \Var_{D_r, \sdp}\E_{D_s|\sdp}[g(x; D_s)] + 
        (\mathrm{SDB} + \mathrm{MB})^2 + \Var_{y}[y],
    \end{equation}
    where
    \begin{align}
        \mathrm{MSE} &= \E_{y, D_r, \sdp, D_s^{1:m}}[(y - \hat{g}(x; D_s^{1:m}))^2] \\
        \mathrm{MV} &= \E_{D_r, \sdp, \theta} \Var_{D_s|\theta}[g(x; D_s)] \\
        \mathrm{SDV} &= \E_{D_r, \sdp}\Var_{\theta | \sdp}\E_{D_s | \theta}[g(x; D_s)] \\
        \mathrm{SDB} &= \E_{D_r, \sdp}[f(x) - \E_{\theta | \sdp}[f_\theta(x)]] \\
        \mathrm{MB} &= \E_{D_r, \sdp, \theta}[f_\theta(x) - \E_{D_s|\theta}[g(x; D_s)]].
    \end{align}
    
    We can additionally decompose 
    \begin{equation}
        \begin{split}
            \Var_{D_r, \sdp}\E_{D_s|\sdp}[g(x; D_s)]
            &= \E_{D_r}\Var_{\sdp|D_r}\E_{D_s|\sdp}[g(x; D_s)]
            + \Var_{D_r}\E_{\sdp|D_r}\E_{D_s|\sdp}[g(x; D_s)]
            \\&= \E_{D_r}\Var_{\sdp|D_r}\E_{D_s|\sdp}[g(x; D_s)]
            + \Var_{D_r}\E_{D_s|D_r}[g(x; D_s)]
        \end{split}
    \end{equation}
    This reveals the DP-related variance term
    \begin{equation}
        \mathrm{DPVAR} = \E_{D_r}\Var_{\sdp|D_r}\E_{D_s|\sdp}[g(x; D_s)]
    \end{equation}
    so we have
    \begin{equation}
        \mathrm{MSE} = \frac{1}{m}\mathrm{MV} + \frac{1}{m}\mathrm{SDV} 
        + \mathrm{RDV} + \mathrm{DPVAR} + (\mathrm{SDB} + \mathrm{MB})^2 + \Var_y[y].
    \end{equation}
\end{proof}

\section{NON-I.I.D.\ SYNTHETIC DATA}\label{sec:mse-decomposition-non-iid}
Here, we consider the case of non-i.i.d\ synthetic datasets, and allow each synthetic dataset to have 
a different predictor $g_i$. We get a similar decomposition as in 
Theorem~\ref{thm:mse-synthetic-data-decomposition}, but the terms of the decomposition are now 
averages over  the possibly different synthetic data distributions, and there is an 
additional covariance term.
\begin{restatable}{theorem}{theoremmsesyntheticdatadecompositionnoniid}\label{thm:mse-synthetic-data-decomposition-non-iid}
    Let the parameters for $m$ generators $\theta_i\sim p(\theta_i|D_r)$, 
    $i=1,\dotsc,m$, be potentially
    non-i.i.d.\ given the real data $D_r$. 
    Let the synthetic datasets be $D_s^{i} \sim p(D_s^i | \theta_i)$, and
    let $\hat{g}(x; D_s^{1:m}) = \frac{1}{m}\sum_{i=1}^m g_i(x; D_s^i)$. Then
    \begin{equation}
        \mathrm{MSE} = \frac{1}{m}\aMV + \frac{1}{m}\aSDV
        + \COV + \aRDV + (\aSDB + \aMB)^2 + \Var_y[y],
        \label{eq:mse-decomposition-mean-model-non-iid}
    \end{equation}
    where
    \begin{align}
        \MSE &= \E_{y, D_r, D_s^{1:m}}[(y - \hat{g}(x; D_s^{1:m}))^2] \\
        \aMV &= \frac{1}{m}\sum_{i=1}^m \E_{D_r, \theta_i} \Var_{D_s^i|\theta_i}[g_i(x; D_s^i)] \\
        \aSDV &= \frac{1}{m}\sum_{i=1}^m \E_{D_r}
        \Var_{\theta_i | D_r}\E_{D_s^i | \theta_i}[g_i(x; D_s^i)] \\
        \COV &= \frac{1}{m^2}\sum_{i\neq j} 
        \E_{D_r}\left[\Cov_{D_s^i, D_s^j | D_r} [g_i(x; D_s^i), g_j(x; D_s^j)]\right] \\
        \aRDV &= \Var_{D_r}\left[\frac{1}{m} \sum_{i=1}^m \E_{D_s^{i}|D_r}[g_i(x; D_s^i)]\right] \\
        \aSDB &= \frac{1}{m}\sum_{i=1}^m \E_{D_r}\left[
        f(x) - \E_{\theta_{i}|D_r}[f_{\theta_i}(x)]\right] \\
        \aMB &= \frac{1}{m}\sum_{i=1}^m\E_{D_r,\theta_{i}}\left[
        f_{\theta_i}(x) - \E_{D_s^i|\theta_i}[g_i(x, D_s^i)]\right].
    \end{align}
    $f(x) = \E_{y}[y]$ is the optimal predictor for real data and
    $f_{\theta_i}$ is the optimal predictor for the synthetic data generating process
    given parameters $\theta_i$. All random quantities are implicitly conditioned on $x$.
\end{restatable}
\begin{proof}
    The classical bias-variance decomposition gives
    \begin{align}
        \E_{y, D_r, D_s^{1:m}}[(y - \hat{g}(x; D_s^{1:m}))^2]
        = (f(x) - \E_{D_r, D_s^{1:m}}[\hat{g}(x; D_s^{1:m})])^2 
        + \Var_{D_r, D_s^{1:m}}[\hat{g}(x; D_s^{1:m})] + \Var_{y}[y].
    \end{align}
    For the bias,
    \begin{equation}
        \begin{split}
            &f(x) - \E_{D_r, D_s^{1:m}}[\hat{g}(x; D_s^{1:m})] 
            \\&= f(x) - \E_{D_r}\E_{D_s^{1:m}|D_r}\left[\frac{1}{m}\sum_{i=1}^m g_i(x, D_s^i)\right]
            \\&= \E_{D_r}\E_{D_s^{1:m}|D_r}\left[\frac{1}{m}\sum_{i=1}^m (f(x) - g_i(x, D_s^i))\right]
            \\&= \E_{D_r}\E_{D_s^{1:m}|D_r}\left[\frac{1}{m}\sum_{i=1}^m (f(x) - f_{\theta_i}(x) 
            + f_{\theta_i}(x) - g_i(x, D_s^i))\right]
            \\&= \E_{D_r}\E_{\theta_{1:m}|D_r}\left[\frac{1}{m}\sum_{i=1}^m 
            (f(x) - f_{\theta_i}(x))\right]
            + \E_{D_r}\E_{D_s^{1:m}|D_r}\left[\frac{1}{m}\sum_{i=1}^m
            (f_{\theta_i}(x) - g_i(x, D_s^i))\right]
            \\&= \E_{D_r}\left[\frac{1}{m}\sum_{i=1}^m 
            (f(x) - \E_{\theta_{i}|D_r}[f_{\theta_i}(x)])\right]
            + \E_{D_r,\theta_{1:m}}\left[\frac{1}{m}\sum_{i=1}^m
            (f_{\theta_i}(x) - \E_{D_s^i|\theta_i}[g_i(x, D_s^i)])\right]
            \\&= \frac{1}{m}\sum_{i=1}^m \E_{D_r}\left[
            f(x) - \E_{\theta_{i}|D_r}[f_{\theta_i}(x)]\right]
            + \frac{1}{m}\sum_{i=1}^m\E_{D_r,\theta_{i}}\left[
            f_{\theta_i}(x) - \E_{D_s^i|\theta_i}[g_i(x, D_s^i)]\right].
        \end{split}
    \end{equation}
    
    For the variance,
    \begin{equation}
        \begin{split}
            \Var_{D_r, D_s^{1:m}}[\hat{g}(x; D_s^{1:m})]
            &= \E_{D_r}\Var_{D_s^{1:m}|D_r}\left[\frac{1}{m} \sum_{i=1}^m g_i(x; D_s^i)\right]
            + \Var_{D_r}\E_{D_s^{1:m}|D_r}\left[\frac{1}{m} \sum_{i=1}^m g_i(x; D_s^i)\right]
            \\&= \frac{1}{m^2} \E_{D_r}\Var_{D_s^{1:m}|D_r}\left[\sum_{i=1}^m g_i(x; D_s^i)\right]
            + \Var_{D_r}\left[\frac{1}{m} \sum_{i=1}^m \E_{D_s^{i}|D_r}[g_i(x; D_s^i)]\right],
            %\\&= \frac{1}{m^2} \E_{D_r}\Var_{D_s^{1:m}|D_r}\left[\sum_{i=1}^m g(x; D_s^i)\right]
            %+ \Var_{D_r}\E_{D_s|D_r}\left[g(x; D_s)\right]
            %\\&= \frac{1}{m} \E_{D_r}\Var_{D_s|D_r}\left[g(x; D_s)\right]
            %+ \Var_{D_r}\E_{D_s|D_r}\left[g(x; D_s)\right],
        \end{split}
    \end{equation}
    \begin{equation}
        \begin{split}
            \Var_{D_s^{1:m} | D_r}\left[\sum_{i=1}^m g_i(x; D_s^i)\right] 
            &= \sum_{i=1}^m \Var_{D_s^i | D_r}[g_i(x; D_s^i)] 
            + \sum_{i\neq j} \Cov_{D_s^i, D_s^j | D_r} [g_i(x; D_s^i), g_j(x; D_s^j)]
            %\\&= m \Var_{D_s | D_r}[g(x; D_s)] 
            %+ m(m-1) \Cov_{D_s^1, D_s^2 | D_r} [g(x; D_s^1), g(x; D_s^2)],
        \end{split}
    \end{equation}
    and
    \begin{equation}
        \Var_{D_s^i|D_r}[g_i(x; D_s^i)] = \E_{\theta_i|D_r}\Var_{D_s^i|\theta_i}[g_i(x; D_s^i)] + \Var_{\theta_i|D_r}\E_{D_s^i|\theta}[g_i(x, D_s^i)].
    \end{equation}
\end{proof}

If the synthetic datasets are identically distributed, but not necessarily independent, and the predictors
are identical, Theorem~\ref{thm:mse-synthetic-data-decomposition-non-iid} simplifies to 
\begin{equation}
    \mathrm{MSE} = \frac{1}{m}\mathrm{MV} + \frac{1}{m}\mathrm{SDV} 
    + \left(1 - \frac{1}{m}\right)\mathrm{COV}
    + \mathrm{RDV}+ (\mathrm{SDB} + \mathrm{MB})^2 + \Var_y[y],
    \label{eq:mse-decomposition-mean-model-with-cov}
\end{equation}
where
\begin{equation}
    \begin{split}
        \mathrm{MSE} &= \E_{y, D_r, D_s^{1:m}}[(y - \hat{g}(x; D_s^{1:m}))^2] \\
        \mathrm{MV} &= \E_{D_r, \theta} \Var_{D_s|\theta}[g(x; D_s)] \\
        \mathrm{SDV} &= \E_{D_r}\Var_{\theta | D_r}\E_{D_s | \theta}[g(x; D_s)] \\
        \mathrm{COV} &= \E_{D_r}\left[\Cov_{D_s^1, D_s^2 | D_r}[g(x; D_s^1), g(x; D_s^2)]\right] \\
        \mathrm{RDV} &= \Var_{D_r}\E_{D_s|D_r}[g(x; D_s)] \\
        \mathrm{SDB} &= \E_{D_r}[f(x) - \E_{\theta | D_r}[f_\theta(x)]] \\
        \mathrm{MB} &= \E_{D_r, \theta}[f_\theta(x) - \E_{D_s|\theta}[g(x; D_s)]].
    \end{split}
\end{equation}
which is Theorem~\ref{thm:mse-synthetic-data-decomposition} with the additional covariance term.

In the noisy summary case, we get an analogue of 
Theorem~\ref{thm:mse-dp-synthetic-data-decomposition}.
\begin{restatable}{theorem}{theoremmsedpsyntheticdatadecompositionnoniid}\label{thm:mse-dp-synthetic-data-decomposition-non-iid}
    Let the parameters for $m$ generators $\theta_i \sim p(\theta_i|\sdp)$, 
    $i=1,\dotsc,m$, be potentially non-i.i.d.\ given a DP summary $\sdp$ the real 
    data $D_r$, let the synthetic datasets be $D_s^{i} \sim p(D_s^i | \theta_i)$, and
    let $\hat{g}(x; D_s^{1:m}) = \frac{1}{m}\sum_{i=1}^m g_i(x; D_s^i)$. Then
    \begin{equation}
        \mathrm{MSE} = \frac{1}{m}\aMV + \frac{1}{m}\aSDV
        + \COV + \aRDV + \aDPVAR + (\aSDB + \aMB)^2 + \Var_y[y],
        \label{eq:mse-dp-decomposition-non-iid}
    \end{equation}
    where
    \begin{align}
        \MSE &= \E_{y, D_r, \sdp, D_s^{1:m}}[(y - \hat{g}(x; D_s^{1:m}))^2] \\
        \aMV &= \frac{1}{m}\sum_{i=1}^m \E_{D_r, \sdp, \theta_i} \Var_{D_s^i|\theta_i}[g_i(x; D_s^i)] \\
        \aSDV &= \frac{1}{m}\sum_{i=1}^m \E_{D_r, \sdp}
        \Var_{\theta_i | \sdp}\E_{D_s^i | \theta_i}[g_i(x; D_s^i)] \\
        \COV &= \frac{1}{m^2}\sum_{i\neq j} 
        \E_{D_r, \sdp}\left[\Cov_{D_s^i, D_s^j | \sdp} [g_i(x; D_s^i), g_j(x; D_s^j)]\right] \\
        \aRDV &= \Var_{D_r}\left[\frac{1}{m} \sum_{i=1}^m \E_{D_s^{i}|D_r}[g_i(x; D_s^i)]\right] \\
        \aDPVAR &= \E_{D_r}\Var_{\sdp|D_r}\left[\frac{1}{m} \sum_{i=1}^m 
        \E_{D_s^{i}|\sdp}[g_i(x; D_s^i)]\right] \\
        \aSDB &= \frac{1}{m}\sum_{i=1}^m \E_{D_r,\sdp}\left[
        f(x) - \E_{\theta_{i}|\sdp}[f_{\theta_i}(x)]\right] \\
        \aMB &= \frac{1}{m}\sum_{i=1}^m\E_{D_r,\sdp,\theta_{i}}\left[
        f_{\theta_i}(x) - \E_{D_s^i|\theta_i}[g_i(x, D_s^i)]\right].
    \end{align}
    $f(x) = \E_{y}[y]$ is the optimal predictor for real data and
    $f_{\theta_i}$ is the optimal predictor for the synthetic data generating process
    given parameters $\theta_i$. All random quantities are implicitly conditioned on $x$.
\end{restatable}
\begin{proof}
    Using $\sdp$ in place of $D_r$ in Theorem~\ref{thm:mse-synthetic-data-decomposition-non-iid}
    gives 
    \begin{equation}
        \MSE = \frac{1}{m}\aMV + \frac{1}{m}\aSDV + \COV 
        + \Var_{D_r,\sdp}\left[\frac{1}{m} \sum_{i=1}^m \E_{D_s^{i}|\sdp}[g_i(x; D_s^i)]\right]
        + (\aSDB + \aMB)^2 + \Var_y[y].
    \end{equation}

    The variance over $D_r$ and $\sdp$ can be decomposed
    \begin{equation}
        \begin{split}
            &\Var_{D_r,\sdp}\left[\frac{1}{m} \sum_{i=1}^m \E_{D_s^{i}|\sdp}[g_i(x; D_s^i)]\right]
            \\&= \E_{D_r}\Var_{\sdp|D_r}\left[\frac{1}{m} \sum_{i=1}^m \E_{D_s^{i}|\sdp}[g_i(x; D_s^i)]\right]
            + \Var_{D_r}\E_{\sdp|D_r}\left[\frac{1}{m} \sum_{i=1}^m \E_{D_s^{i}|\sdp}[g_i(x; D_s^i)]\right]
            \\&= \E_{D_r}\Var_{\sdp|D_r}\left[\frac{1}{m} \sum_{i=1}^m \E_{D_s^{i}|\sdp}[g_i(x; D_s^i)]\right]
            + \Var_{D_r}\left[\frac{1}{m} \sum_{i=1}^m \E_{D_s^{i}|D_r}[g_i(x; D_s^i)]\right].
        \end{split}
    \end{equation}
\end{proof}

\paragraph{Implication of Non-i.i.d.\, Theory}
From Theorems~\ref{thm:mse-synthetic-data-decomposition-non-iid} and \ref{thm:mse-dp-synthetic-data-decomposition-non-iid},
it is clear that the implications of the i.i.d. theory 
do not always hold. For example, when increasing the 
number of synthetic datasets with each generator
being a different model, all of the terms in
Theorems~\ref{thm:mse-synthetic-data-decomposition-non-iid} and \ref{thm:mse-dp-synthetic-data-decomposition-non-iid}
that are averages over the generators can change.

However, we can recover some implications with additional
assumptions. For example, if we assume that the 
generators and downstream predictors are always the 
same, possibly correlated, but the covariance 
term $\COV$ does depend on the number of synthetic datasets,
we can derive a similar MSE prediction rule as in
Section~\ref{sec:estimating-effect-multiple-syn-datasets}
from Equation~\eqref{eq:mse-decomposition-mean-model-with-cov}.

\section{BIAS-VARIANCE DECOMPOSITION FOR BREGMAN DIVERGENCES}\label{app:bregman-divergence-decomposition}

\subsection{Background: Bregman Divergences}
A Bregman divergence~\citep{bregmanRelaxationMethodFinding1967} 
$D_F\colon \R^d\times \R^d \to \R$ is a loss function 
\begin{equation}
    D_F(y, g) = F(y) - F(g) - \nabla F(g)^T (y - g)
\end{equation}
where $F\colon \R^d \to \R$ is a strictly convex differentiable function. Many common error metrics,
like MSE and cross entropy, can be expressed as expected values of a Bregman divergence.
In fact, proper scoring rules\footnote{
    Proper scoring rules are error metrics that are minimised by predicting the correct 
    probabilities.
} 
can be characterised via Bregman 
divergences~\citep{gneitingStrictlyProperScoring2007,kimparaProperLossesDiscrete2023}.
Table~\ref{tab:bregman-divergences} shows how the metrics we consider are 
expressed as Bregman divergences~\citep{guptaEnsemblesClassifiersBiasVariance2022}.

\citet{pfauGeneralizedBiasvarianceDecomposition2013} derive the following bias-variance decomposition for 
Bregman divergences:
\begin{equation}
    \underbrace{\E[D(y, g)]}_{\mathrm{Error}} = \underbrace{\E[D(y, \E y)]}_{\mathrm{Noise}} 
    + \underbrace{D(\E y, \CP g)}_{\mathrm{Bias}} 
    + \underbrace{\E D(\CP g, g)}_{\mathrm{Variance}}
    \label{eq:bregman-ensemble-decomposition}
\end{equation}
$y$ is the true value, and $g$ is the predicted value. All of the random quantities are conditioned on 
$x$.
$\CP$ is a \emph{central prediction}:
\begin{equation}
    \CP g = \argmin_{z} \E D(z, g).
\end{equation}
The variance term can be used to define a generalisation of variance:
\begin{equation}
    \VP g = \E D(\CP g, g)
\end{equation}
$\CP$ and $\VP$ can also be defined conditionally on some random variable $Z$
by making the expectations conditional on $Z$ in the definitions.
These obey generalised laws of total expectation and 
variance~\citep{guptaEnsemblesClassifiersBiasVariance2022}:
\begin{equation}
    \CP g = \CP_Z[\CP_{g|Z}[g]]
\end{equation}
and
\begin{equation}
    \VP g = \E_Z[\VP_{g|Z}[g]] + \VP_Z[\CP_{g|Z}[g]].
\end{equation}

The convex dual of $g$ is $g^* = \nabla F(g)$.
The central prediction $\CP g$ can also be expressed as an expectation 
over the convex dual~\citep{guptaEnsemblesClassifiersBiasVariance2022}:
\begin{equation}
    \CP g = (\E g^*)^* %= \nabla F^*(\E \nabla F(g))
\end{equation}

\begin{table*}%[t]
\caption{Common error metrics as Bregman divergences. $g$ denotes a prediction
in regression and $p$ denotes predicted class probabilities in classification.
$g^{(j)}$ and $p^{(j)}$ denote the predictions of different ensemble members.
$y$ is the correct value in regression, and a one-hot encoding of the correct 
class in classification. The binary classification Brier score only looks at
probabilities for one class. If the multiclass Brier score is used with two 
classes, it is twice the binary Brier score.}
\label{tab:bregman-divergences}
\vskip 0.15in
\begin{center}
\begin{small}
\begin{sc}
\begin{tabular}{lccr}
    \toprule
    Error Metric & $D_F$ & $F(t)$ & Dual Average \\
    \midrule
    MSE & $(y - g)^2$ & $t^2$ & $\frac{1}{m}\sum_{j=1}^m g^{(j)}$ \\
    Brier Score (2 classes) & $(y_0 - p_0)^2$ & $t^2$ & $\frac{1}{m}\sum_{j=1}^m p^{(j)}$\\ 
    Brier Score (Multiclass) & $\sum_{i}(y_i - p_i)^2$ & $\sum_i t_i^2$ & $\frac{1}{m}\sum_{j=1}^m p^{(j)}$\\ 
    Cross Entropy & $-\sum_{i} y_i \ln p_i$ & $\sum_{i} t_i \ln t_i$ & 
    $\mathrm{softmax}\left(\frac{1}{m}\sum_{j=1}^m \ln p^{(j)}\right)$\\
    \bottomrule
\end{tabular}
\end{sc}
\end{small}
\end{center}
\vskip -0.1in
\end{table*}

\citet{guptaEnsemblesClassifiersBiasVariance2022} study the bias-variance 
decomposition of Bregman divergence on a generic ensemble. They show that 
if the ensemble aggregates prediction by averaging them, bias is not preserved,
and can increase. As a solution, they consider \emph{dual averaging},
that is 
\begin{equation}
    \hat{g} = \left(\frac{1}{m}\sum_{i=1}^m g_i^*\right)^*
\end{equation}
for models $g_1, \dotsc, g_m$ forming the ensemble $\hat{g}$.
They show that the bias is preserved in the dual averaged ensemble, and derive a
bias-variance decomposition for them. For mean squared error, the dual average is 
simply the standard average, but for cross entropy, it corresponds to averaging
log probabilities.

\subsection{Bregman Divergence Decomposition for Synthetic Data}\label{sec:bregman-decomposition}

We extend the Bregman divergence decomposition for ensembles from \citet{guptaEnsemblesClassifiersBiasVariance2022}
to generative ensembles.
To prove Theorem~\ref{thm:bregman-synthetic-data-decomposition}, we 
use the following lemma.
\begin{lemma}[\citealt{guptaEnsemblesClassifiersBiasVariance2022}, Proposition 5.3]\label{lemma:iid-dual-ensemble-mean-variance}
    Let $X_1, \dotsc, X_m$ be i.i.d.\  random variables and let 
    $\hat{X} = (\sum_{i=1}^m X_i^*)^*$ be their dual average.
    Then $\CP \hat{X} = \CP X$, $\VP \hat{X} \leq \VP X$ and for any independent $Y$,
    $D(\E Y, \CP \hat{X}) = D(\E Y, \CP X)$.
\end{lemma}

\begin{restatable}{theorem}{theorembregmansyntheticdatadecomposition}\label{thm:bregman-synthetic-data-decomposition}
    When the synthetic datasets $D_s^{1:m}$ are i.i.d.\  given the real data $D_r$ and 
    $\hat{g}(x; D_s^{1:m}) = (\frac{1}{m}\sum_{i=1}^m g(x; D_s^i)^*)^*$,
    \begin{equation}
        \mathrm{Error} \leq \mathrm{MV} + \mathrm{SDV} + \mathrm{RDV} + \mathrm{Bias} + \mathrm{Noise} 
    \end{equation}
    where 
    \begin{align}
        \mathrm{Error} &= \E_{y, D_r, D_s^{1:m}}[D(y, \hat{g})] \\
        \mathrm{MV} &= \E_{D_r}\E_{\theta | D_r}\VP_{D_s | \theta}[g] \\
        \mathrm{SDV} &= \E_{D_r}\VP_{\theta | D_r}\CP_{D_s | \theta}[g] \\
        \mathrm{RDV} &= \VP_{D_r}\CP_{D_s|D_r}[g] \\
        \mathrm{Bias} &= D\left(\E_{y} y, \CP_{D_r}\CP_{D_s|D_r} [g]\right)\\
        \mathrm{Noise} &= \E_{y}\left[D(y, \E_y y)\right]
    \end{align}
\end{restatable}
\begin{proof}
    Plugging the ensemble $\hat{g}$ into the decomposition 
    \eqref{eq:bregman-ensemble-decomposition} gives 
    \begin{equation}
        \E_{y, D_r, D_s^{1:m}}[D(y, \hat{g})] 
        = \E_{y}[D(y, \E_{y} y)] 
        + D(\E_{y} y, \CP_{D_r,D_s^{1:m}} \hat{g}) 
        + \VP_{D_r,D_s^{1:m}} [\hat{g}]
    \end{equation}

    Applying the generalised laws of expectation and variance, and 
    Lemma~\ref{lemma:iid-dual-ensemble-mean-variance} to the variance term, 
    we obtain:
    \begin{align}
        \VP_{D_r, D_s^{1:m}} [\hat{g}] 
        &= \E_{D_r}\VP_{D_s^{1:m}|D_r}[\hat{g}] + \VP_{D_r}\CP_{D_s^{1:m}|D_r}[\hat{g}]
    \end{align}
    For the second term on the right:
    \begin{equation}
        \CP_{D_s^{1:m}|D_r}[\hat{g}] = \CP_{D_s|D_r}[g],
    \end{equation}
    which gives the RDV:
    \begin{equation}
        \VP_{D_r}\CP_{D_s^{1:m}|D_r}[\hat{g}] = \VP_{D_r}\CP_{D_s|D_r}[g].
    \end{equation}
    For the first term on the right:
    \begin{equation}
        \E_{D_r}\VP_{D_s^{1:m}|D_r}[\hat{g}] \leq \E_{D_r}\VP_{D_s|D_r}[g]
    \end{equation}
    and
    \begin{equation}
        \VP_{D_s|D_r}[g] = \E_{\theta | D_r}\VP_{D_s | \theta}[g] 
        + \VP_{\theta | D_r}\CP_{D_s | \theta}[g],
    \end{equation}
    which give MV and SDV. 
    
    For the bias
    \begin{equation}
        \begin{split}
            D(\E_{y} y, \CP_{D_r, D_s^{1:m}} [\hat{g}]) 
            &= D(\E_{y} y, \CP_{D_r}\CP_{D_s^{1:m}|D_r} [\hat{g}])
            \\&= D(\E_{y} y, \CP_{D_r}\CP_{D_s|D_r} [g]).
        \end{split}
    \end{equation}
    Putting everything together proves the claim.
\end{proof}
This decomposition is not as informative as the other two in 
Section~\ref{sec:mse-decomposition}, as is only gives 
an upper bound, and does not explicitly depend on the number of synthetic datasets.

\paragraph{Implication for Non-Synthetic Data Ensembles}
The theory of \citet{guptaEnsemblesClassifiersBiasVariance2022}
assumes that each member of the ensemble is trained with an 
independently sampled real dataset. In practice, this would 
mean that one needs to split the training data between each 
ensemble member to apply their theory, so their theory 
does not apply to any of the ways ensembles are 
usually trained.
In contrast, our theory applies to bagging, as discussed in Section 2.1, so our 
Theorem~\ref{thm:bregman-synthetic-data-decomposition} implies a generalisation of Proposition 5.3 of 
\citet{guptaEnsemblesClassifiersBiasVariance2022} to bagging.

\section{EXPERIMENTAL DETAILS}\label{sec:experiment-details}

\subsection{Datasets}\label{sec:dataset-details}
In our experiments, we use 7 tabular datasets. For four of them, the downstream prediction 
task is regression, and for the other three, the prediction task is binary classification.
Table~\ref{tab:dataset-information} lists some general information on the datasets.
We use 25\% of the real data as a test set, with the remaining 75\% being used to generate
the synthetic data, for all of the datasets. All experiments are repeated several times,
with different randomly selected train-test splits for each repeat.

\begin{table}%[t]
\caption{Details on the datasets used in the experiments. \# Cat. and \# Num. are the 
numbers of categorical and numerical features, not counting the target variable. For datasets
with removed rows, the table shows the number of rows after the removals.}
\label{tab:dataset-information}
\vskip 0.15in
\begin{center}
\begin{small}
\begin{sc}
\begin{tabular}{lcccr}
    \toprule
    Dataset & \# Rows & \# Cat. & \# Num. & Task \\
    \midrule
    Abalone & 4177 & 1 & 7 & Regression \\
    ACS 2018 & 50000 & 5 & 2 & Regression \\
    Adult & 45222 & 8 & 4 & Classification \\
    Reduced Adult & 46043 & 7 & 2 & Classification \\
    Breast Cancer & 569 & 0 & 30 & Classification \\
    California Housing & 20622 & 0 & 8 & Regression \\
    German Credit & 1000 & 13 & 7 & Classification \\
    Insurance & 1338 & 3 & 3 & Regression \\
    \bottomrule
\end{tabular}
\end{sc}
\end{small}
\end{center}
\vskip -0.1in
\end{table}

\paragraph{Abalone} \citep[CC BY 4.0]{nashAbalone1995} The abalone dataset contains information on abalones,
with the task of predicting the number of rings on the abalone from other information like weight and size.

\paragraph{ACS 2018}  
(\url{https://www.census.gov/programs-surveys/acs/microdata/documentation.2018.html}, license:
\url{https://www.census.gov/data/developers/about/terms-of-service.html})
This dataset contains several variables from the American community survey (ACS) of 2018, 
with the task of predicting a person's income from the other features. Specifically, the variables we selected are 
AGEP (age), COW (employer type), SCHL (education), MAR (marital status), WKHP (working hours), 
SEX, RAC1P (race), and the target PINCP (income). We take a subset of 50000 datapoints 
from the California data, and log-transform the target variable. We used the 
folktables package~\citep{dingRetiringAdultNew2021} to download the given subset of the 
data.

\paragraph{Adult} \citep[CC BY 4.0]{kohaviAdult1996}
The UCI Adult dataset contains general information on people, with the task of predicting 
whether their income is over \$50000. We drop rows with any missing values.

\paragraph{Reduced Adult} \citep[CC BY 4.0]{kohaviAdult1996}
This dataset is the UCI Adult dataset with a reduced set of features. The subset of features is 
age, workclass, education, marital-status, race, gender, capital-gain, capital-loss
and hours-per-week.\footnote{This subset was used in the NAPSU-MQ experiments 
of \citet{raisaNoiseawareStatisticalInference2023}} 
We binarise capital-gain and capital-loss to indicate whether the original 
value is positive or not for both synthetic data generation and downstream prediction. 
We discretise age and hours-per-week to 5 categories for synthetic data generation, and 
converted back to continuous values for the downstream prediction, as our differentially private synthetic data generators 
require discrete data. We only remove rows with missing values in the included columns, so the 
number of rows is larger with the reduced set of features.

\paragraph{Breast Cancer} \citep[CC BY 4.0]{wolbergBreastCancerWisconsin1995}
The breast cancer dataset contains features derived from images of potential tumors,
with the task of predicting whether the potential tumor is benign or malignant.

\paragraph{California Housing} (\url{https://scikit-learn.org/stable/datasets/real_world.html#california-housing-dataset}, license unknown)
The california housing dataset contains information on housing districts, specifically census block groups, 
in California. The task is predicting the median house value in the district. We removed outlier rows where the 
average number of rooms is at least 50, or the average occupancy is at least 30. According to 
the dataset description, these likely correspond to districts with many empty houses. We 
log-transformed the target variable, as well as the population and median income features.

\paragraph{German Credit} \citep[CC BY 4.0]{hofmannStatlogGermanCredit1994}
The German credit dataset contains information on a bank's customers, with the task of predicting
whether the customers are ``good'' or ``bad''.

\paragraph{Insurance}
(\url{https://www.kaggle.com/datasets/mirichoi0218/insurance/data}, Database Contents License (DbCL) v1.0)
The insurance dataset contains general information on people, like age, gender and BMI, as 
well as the amount they charged their medical insurance, which is the variable to predict.
We take a log transform of the target variable before generating synthetic data.

\subsection{Downstream Prediction Algorithms}\label{sec:downstream-model-details}

We use the scikit-learn\footnote{\url{https://scikit-learn.org/stable/index.html}, BSD 3-Clause License} implementations
of all of the downstream algorithms, which includes probability predictions for all algorithms on the 
classification tasks. We standardise the data before training for all downstream 
algorithms except the tree-based algorithms, specifically decision tree, random forest, and gradient 
boosted trees. This standardisation is done just before downstream training, so the input to the 
synthetic data generation algorithms is not standardised. We use the default hyperparameters of 
scikit-learn for all downstream algorithms except MLP, where we increased the maximum number of iterations
to 1000, as the default was not enough to converge on some datasets. In particular, this means that 
decision trees are trained to interpolate the training data, resulting in high variance of the predictions.

\subsection{DP Experiment}
Both AIM~\citep{mckennaAIMAdaptiveIterative2022} and NAPSU-MQ~\citep{raisaNoiseawareStatisticalInference2023}
generate synthetic data based on noisy values of marginal queries on the real data, which count how many 
rows of the data have given values for given variables. AIM includes a mechanism that chooses a subset of 
queries to measure under DP from a potentially very large workload. We set the workload to be all 
marginal queries over two variables. NAPSU-MQ does not include a query selection mechanism, so we first 
run AIM with $\epsilon = 0.5$, $\delta = \frac{1}{2}n^{-2}$ to select a subset of marginals and then run 
NAPSU-MQ with the selected marginals with $\epsilon = 1$, $\delta = \frac{1}{2}n^{-2}$, which results 
in the same privacy 
bounds~\citep[$\epsilon = 1.5$, $\delta = n^{-2}$;][]{dworkAlgorithmicFoundationsDifferential2014} 
that we used for AIM.
We split the privacy budget between the $m$ synthetic datasets in AIM by dividing the zero-concentrated 
DP~\citep{bunConcentratedDifferentialPrivacy2016} parameter that AIM uses internally by $m$, which keeps 
the total privacy budget fixed.

We used the implementation of the authors for 
AIM\footnote{\url{https://github.com/ryan112358/private-pgm/blob/master/mechanisms/aim.py}, Apache-2.0 license}, and 
used the Twinify library\footnote{\url{https://github.com/DPBayes/twinify}, Apache-2.0 license} for NAPSU-MQ.
We used the default hyperparameters for AIM when generating synthetic data. For selecting the queries for 
NAPSU-MQ, we set the maximum junction tree size\footnote{The junction tree size of the selected queries 
determines how difficult the selected queries are the probabilistic graphical model algorithms NAPSU-MQ and 
AIM use. Their runtime is roughly linear in the junction tree size~\citep{mckennaAIMAdaptiveIterative2022}.} 
hyperparameter to \qty{0.001}{MiB} to ensure that 
NAPSU-MQ does not run for an unreasonable amount of time with the selected queries. The default is 
\qty{80}{MiB}, so the synthetic data generated by AIM is based on a much more comprehensive set of queries
than the synthetic data from NAPSU-MQ. For posterior inference in NAPSU-MQ, we used MCMC with 1000 kept 
samples, 500 warmup samples, and 2 chains.

\subsection{Computational Resources}\label{sec:compute-resources}
We ran the synthetic data generation algorithms DDPM, TVAE and CTGAN on a single GPU, synthpop, AIM and NAPSU-MQ on CPU, and all downstream analysis on CPU, all in a cluster environment.

\section{EXTRA RESULTS}\label{sec:extra-results}

\subsection{Synthetic Data Generation Algorithms}\label{sec:synthetic-data-algorithms-comparison}
We compare several synthetic data generation algorithms to see which algorithms are most 
interesting for subsequent experiments. We use the California housing dataset, where the downstream
task is regression. The algorithms we compare are 
DDPM~\citep{kotelnikovTabDDPMModellingTabular2023}, TVAE~\citep{xuModelingTabularData2019}, 
CTGAN~\citep{xuModelingTabularData2019} and synthpop~\citep{nowokSynthpopBespokeCreation2016}. 
DDPM, TVAE and CTGAN are a diffusion model, 
a variational autoencoder and a GAN that are designed for tabular data. We use the implementations from the
synthcity library\footnote{\url{https://github.com/vanderschaarlab/synthcity},  Apache-2.0 license} for these. 
Synthpop generates synthetic data by sampling one 
column from the real data, and generating the other columns by sequentially training a predictive model on the 
real data, and predicting the next column from the already generated ones. We use the implementation
from the authors~\citep{nowokSynthpopBespokeCreation2016}.\footnote{License: GPL-3}

We use the default hyperparameters 
for all of the algorithms. Synthpop and DDPM have a setting that could potentially affect the randomness in
the synthetic data generation, so we include both possibilities for these settings in this experiment.
For synthpop, this setting is whether the synthetic data is generated from a Bayesian posterior predictive
distribution, which synthpop calls ``proper'' synthetic data. For DDPM, this setting is whether the loss 
function is MSE or KL divergence. In the plots, the two variants of synthpop are called ``SP-P'' and ``SP-IP'' for the 
proper and improper variants, and the variants of DDPM are ``DDPM'' and ``DDPM-KL''.

The results of the comparison are shown in Figure~\ref{fig:synthetic-data-algo-comparison}. We see that synthpop and 
DDPM with MSE loss generally outperform the other generation algorithms, so we select them for 
the subsequent experiments. There is very little difference between the two variants of synthpop,
so we choose the ``proper'' variant due to its connection with the Bayesian reasoning for using 
multiple synthetic datasets.
 
\begin{figure*}
    \centering
    \includegraphics[width=\textwidth]{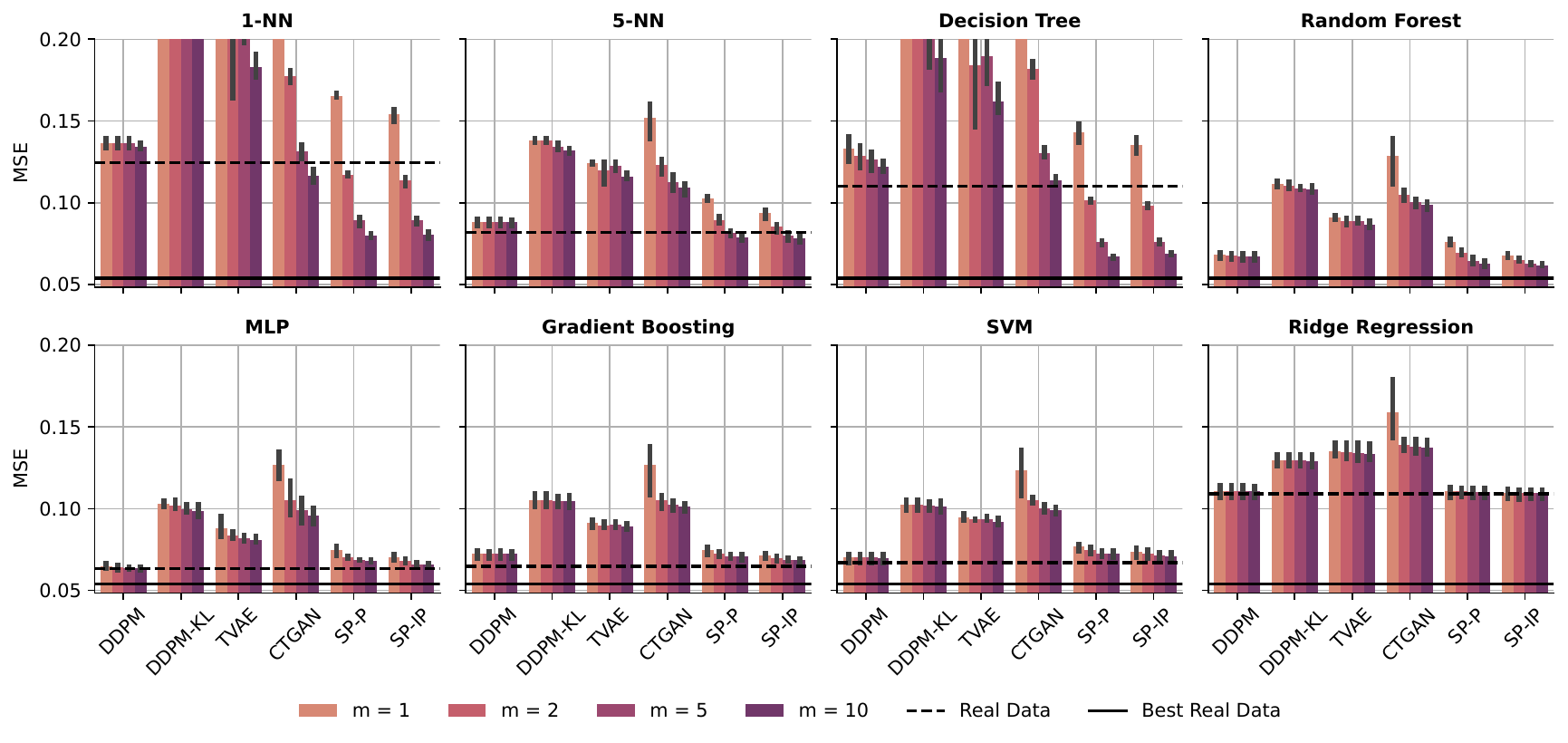}
    \caption{
        Comparison of synthetic data generation algorithms for several prediction algorithms on
        the California housing dataset, with 1 to 10 synthetic datasets. 
        DDPM and synthpop achieve smaller MSE in the downstream predictions, 
        so they were selected for further experiments. SP-P and SP-IP are the proper and improper 
        variants of synthpop, and DDPM-KL is DDPM with KL divergence loss. 1-NN and 5-NN are 
        nearest neighbours with 1 and 5 neighbours.
        The dashed black lines show the performance of each 
        prediction algorithm on the real data, and the solid black line shows the performance of the best 
        predictor, random forest, on the real data. The results are averaged over 3 repeats, with different train-test splits. 
        The error bars are 95\% confidence intervals formed by bootstrapping
        over the repeats. Linear regression was omitted, as it had nearly identical results as ridge regression. Table~\ref{table:synthetic-data-algo-comparison} in the Appendix contains 
        the numbers in the plots, including ridge regression.
    }
    \label{fig:synthetic-data-algo-comparison}
\end{figure*}

\FloatBarrier
\newpage
\subsection{Extra Plots}\label{sec:extra-plots}

\begin{figure}[b]
    \centering
    \includegraphics[width=\textwidth]{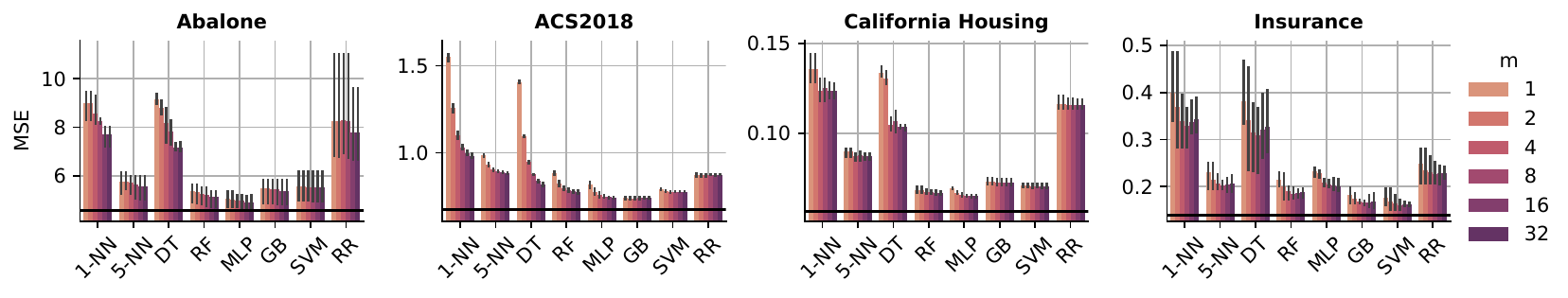}
    \caption{
        MSE on regression datasets of the ensemble of downstream predictors, with varying 
        number of synthetic datasets $m$ from DDPM. Increasing the number of 
        synthetic datasets generally decreases MSE, especially for decision trees and 1-NN.
        The predictors are nearest neighbours 
        with 1 or 5 neighbours (1-NN and 5-NN), decision tree (DT), random forest (RF), a multilayer perceptron (MLP), 
        gradient boosted trees (GB), a support vector machine (SVM) and ridge regression (RR). 
        The black line is the MSE of the best predictor on real data. 
        The results are averaged over 3 repeats. The error bars are 95\% confidence intervals formed 
        by bootstrapping over the repeats.
        We omitted linear regression from the plots, as it had almost identical results to
        ridge regression.
        Tables~\ref{table:abalone-results} to 
        \ref{table:insurance-results} 
        contain the numbers from the plots, including linear regression.
    }
    \label{fig:ddpm-regression-results}
\end{figure}

\begin{figure*}
    \centering
    \includegraphics[width=\textwidth]{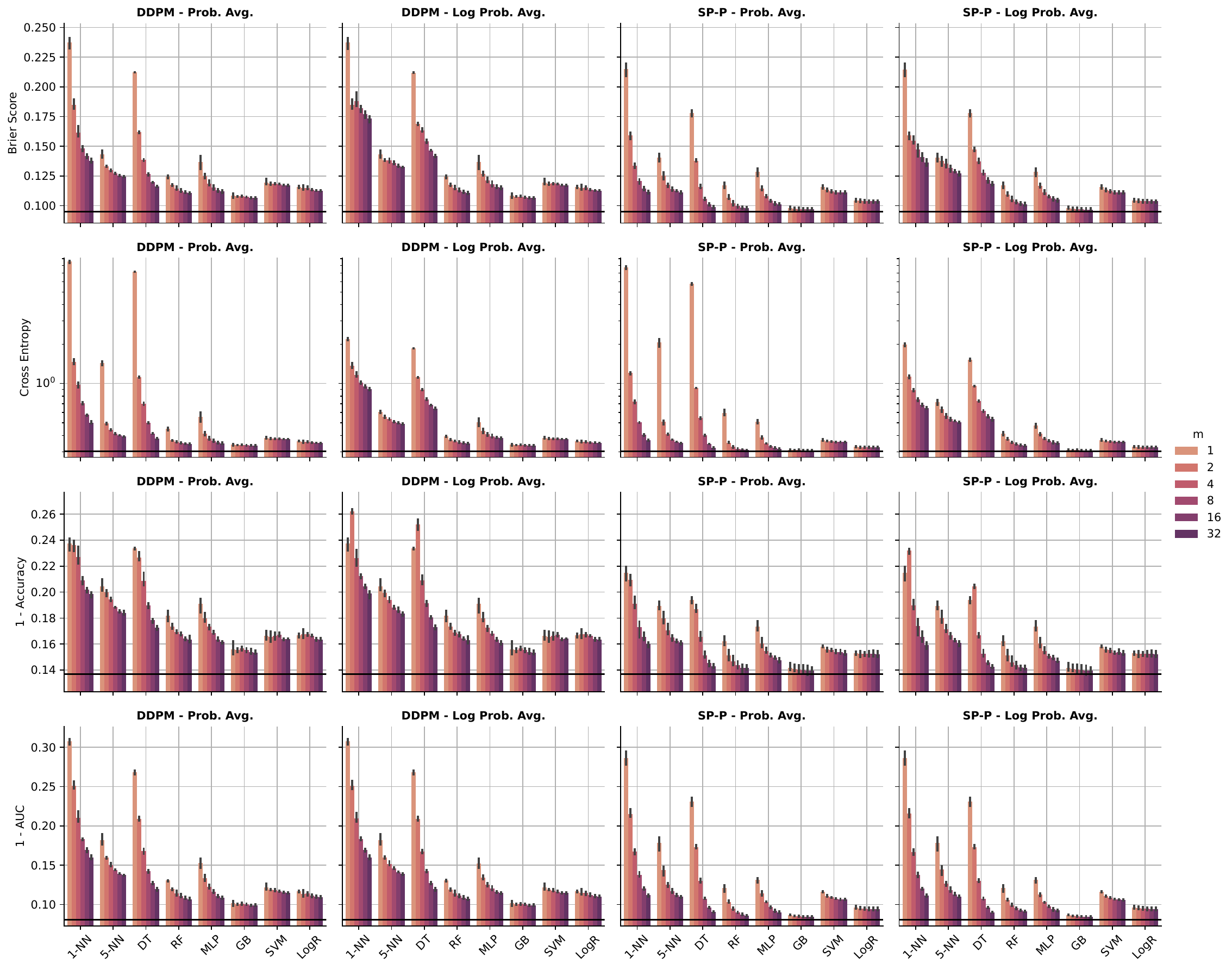}
    \caption{
        All error metrics on the Adult dataset. Note the logarithmic scale on the 
        cross entropy y-axis. The predictors are nearest neighbours 
        with 1 or 5 neighbours (1-NN and 5-NN), decision tree (DT), random forest (RF), a multilayer perceptron (MLP), 
        gradient boosted trees (GB), a support vector machine (SVM) and logistic regression (LogR).
        The black line show the loss of the best 
        downstream predictor trained on real data. The results are averaged over 3 repeats with 
        different train-test splits. The error bars are 95\% confidence intervals formed by 
        bootstrapping over the repeats.
    }
    \label{fig:adult-results}
\end{figure*}

\begin{figure*}
    \centering
    \includegraphics[width=\textwidth]{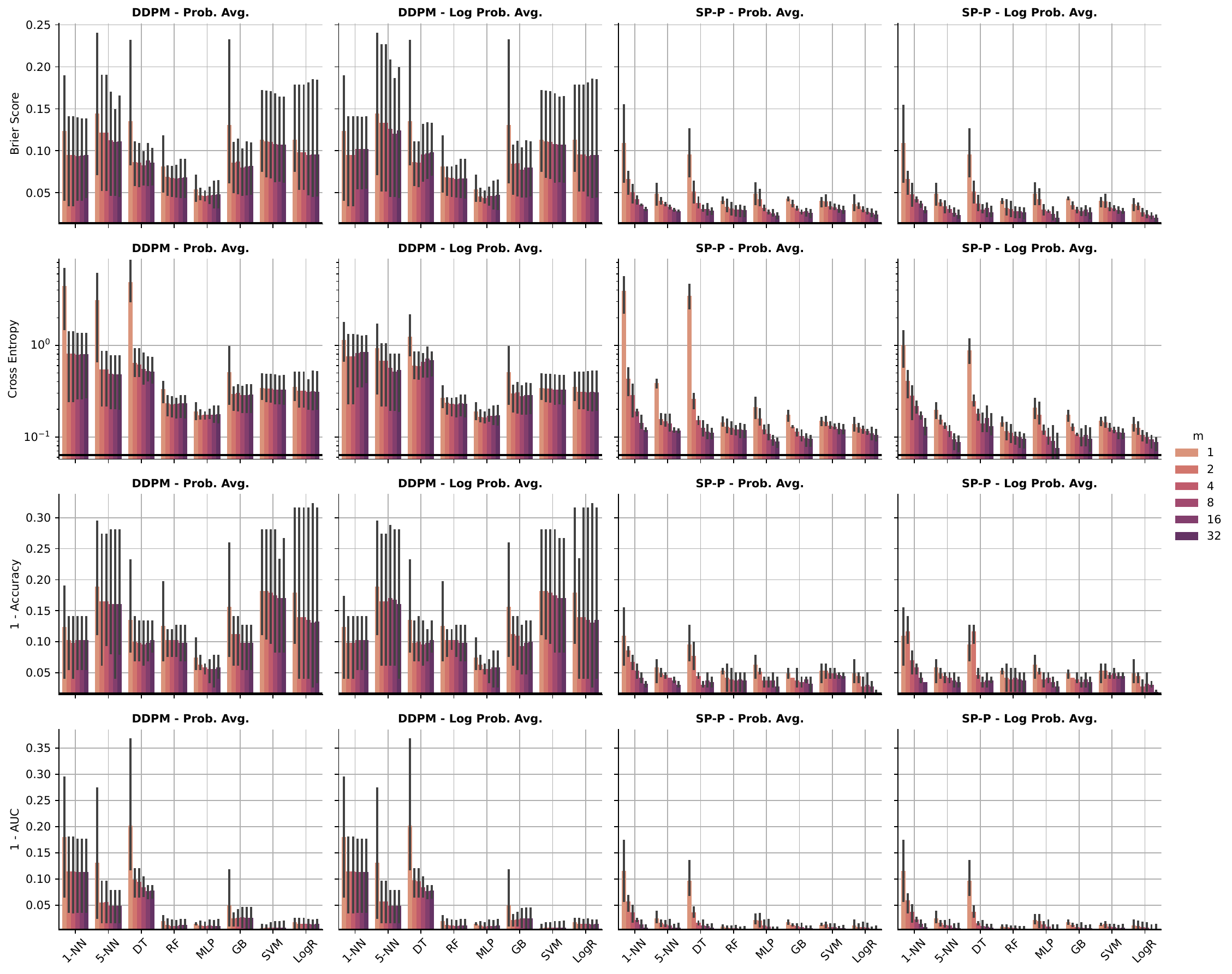}
    \caption{
        All error metrics on the breast cancer dataset. Note the he logarithmic scale on the 
        cross entropy y-axis. The predictors are nearest neighbours 
        with 1 or 5 neighbours (1-NN and 5-NN), decision tree (DT), random forest (RF), a multilayer perceptron (MLP), 
        gradient boosted trees (GB), a support vector machine (SVM) and logistic regression (LogR).
        The black line show the loss of the best 
        downstream predictor trained on real data. The results are averaged over 3 repeats with 
        different train-test splits. The error bars are 95\% confidence intervals formed by 
        bootstrapping over the repeats.
    }
    \label{fig:breast-cancer-results}
\end{figure*}

\begin{figure*}
    \centering
    \includegraphics[width=\textwidth]{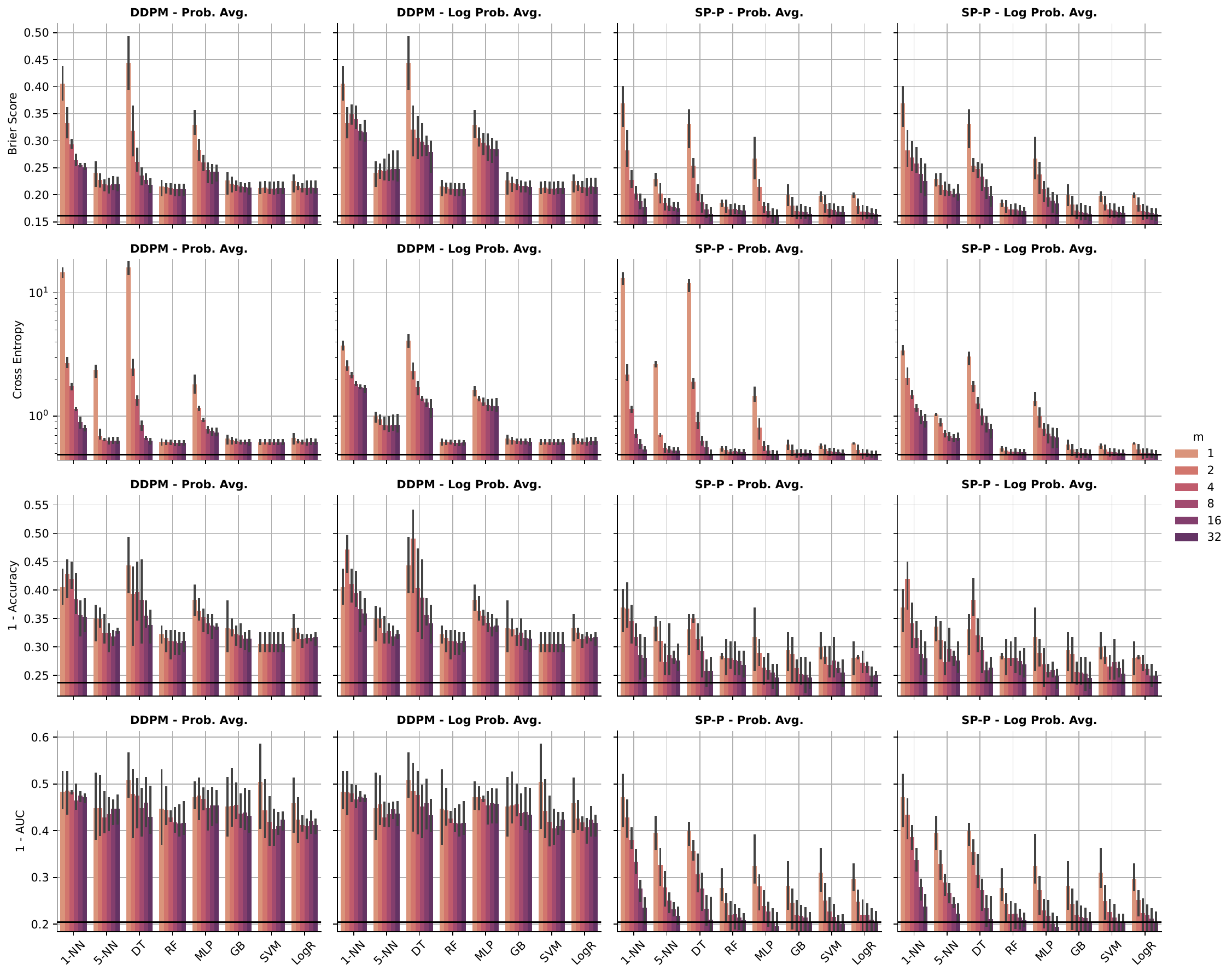}
    \caption{
        All error metrics on the German credit dataset. Note the he logarithmic scale on the 
        cross entropy y-axis. The predictors are nearest neighbours 
        with 1 or 5 neighbours (1-NN and 5-NN), decision tree (DT), random forest (RF), a multilayer perceptron (MLP), 
        gradient boosted trees (GB), a support vector machine (SVM) and logistic regression (LogR).
        The black line show the loss of the best 
        downstream predictor trained on real data. The results are averaged over 3 repeats with 
        different train-test splits. The error bars are 95\% confidence intervals formed by 
        bootstrapping over the repeats.
    }
    \label{fig:german-credit-results}
\end{figure*}

\begin{figure*}
    \begin{subfigure}{\textwidth}
        \centering
        \includegraphics[width=\textwidth]{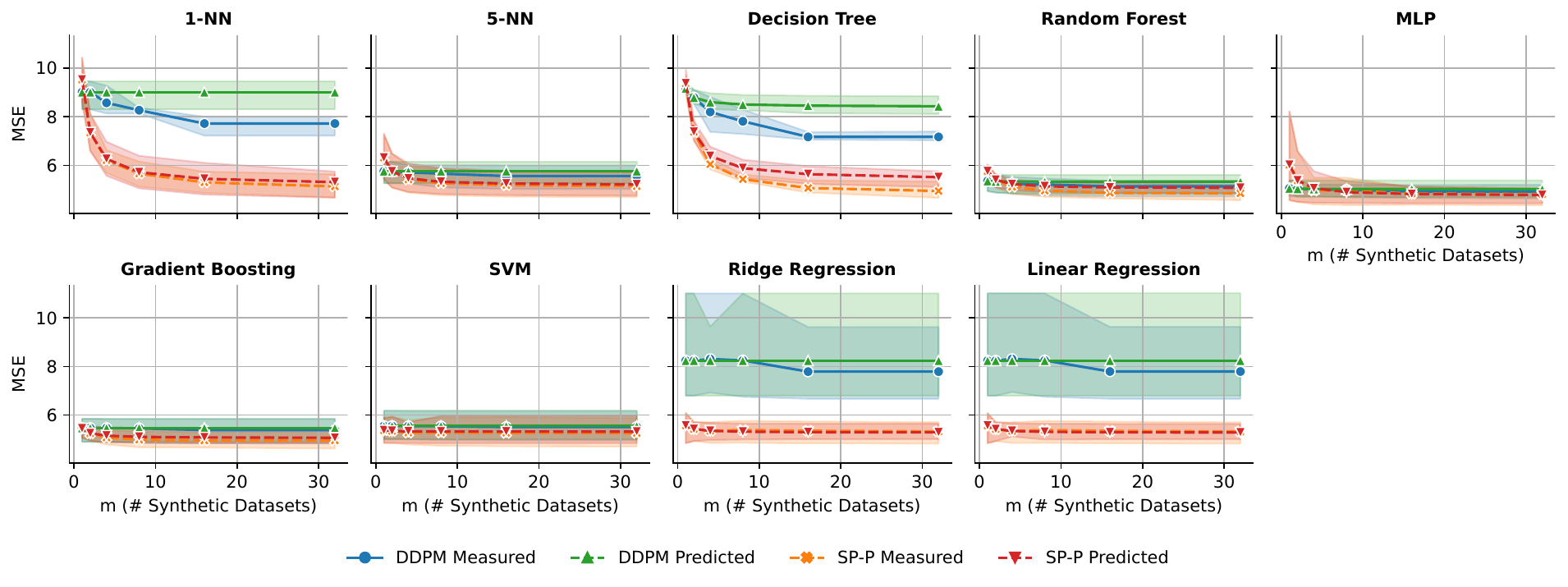}
        \caption{Abalone}
    \end{subfigure}
    \begin{subfigure}{\textwidth}
        \centering
        \includegraphics[width=\textwidth]{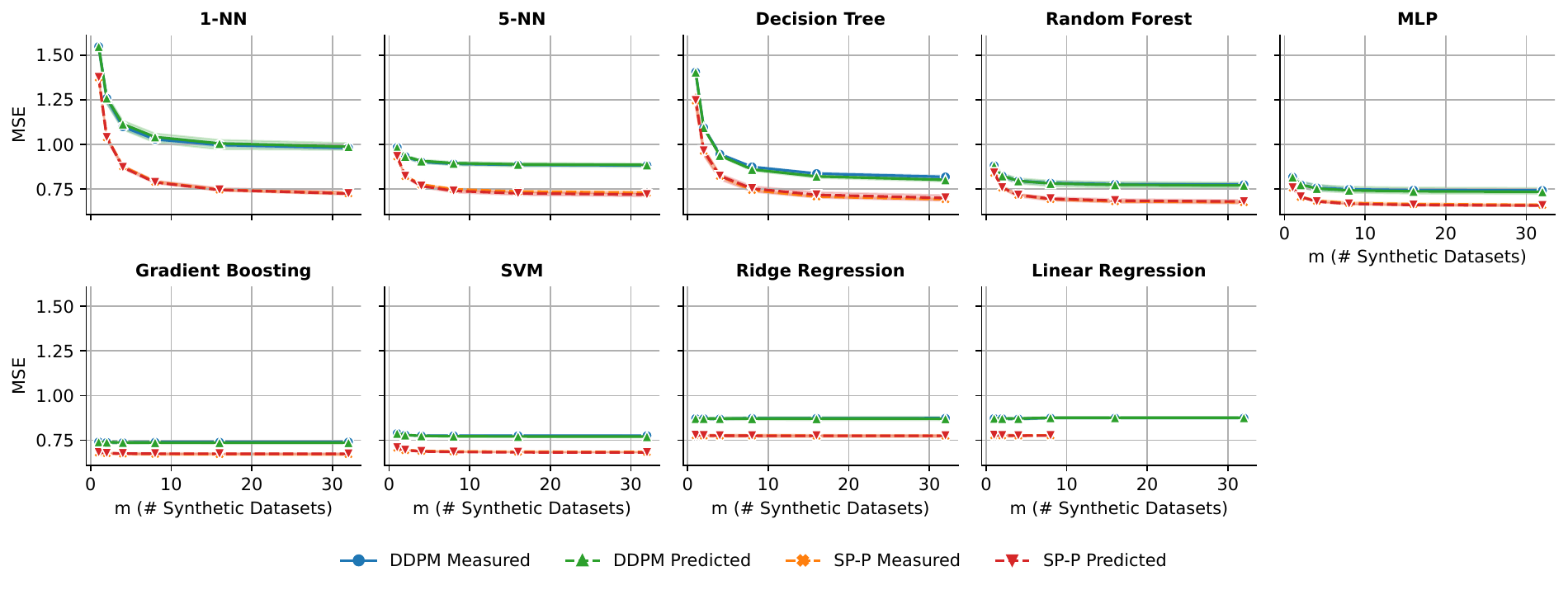}
        \caption{ACS 2018}
        \label{fig:mse-prediction-acs}
    \end{subfigure}
    \caption{
        MSE prediction on the first two regression datasets. 
        The predictions for synthpop are very accurate, and
        the predictions for DDPM are accurate for most cases.
        The linear regression measured MSE 
        line for synthpop with ACS 2018 data is cut off due to excluding repeats with extremely 
        large MSE ($\geq 10^6$).
        1-NN and 5-NN are nearest neighbours with 1 or 5 neighbours. The results are averaged over 3 
        repeats with different train-test splits. The error bands are 95\% confidence 
        intervals formed by bootstrapping over the repeats. 
    }
    \label{fig:mse-prediction-regression1}
\end{figure*}
\begin{figure*}
    \begin{subfigure}{\textwidth}
        \centering
        \includegraphics[width=\textwidth]{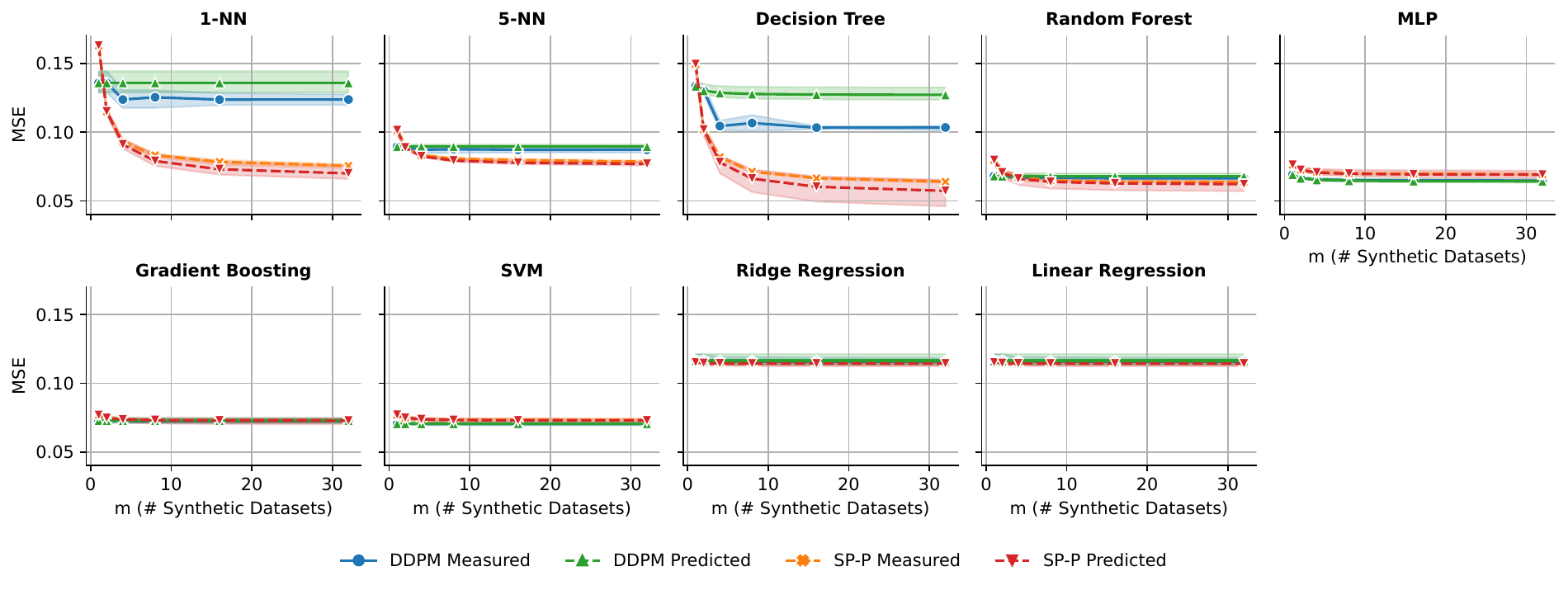}
        \caption{California Housing}
    \end{subfigure}
    \begin{subfigure}{\textwidth}
        \centering
        \includegraphics[width=\textwidth]{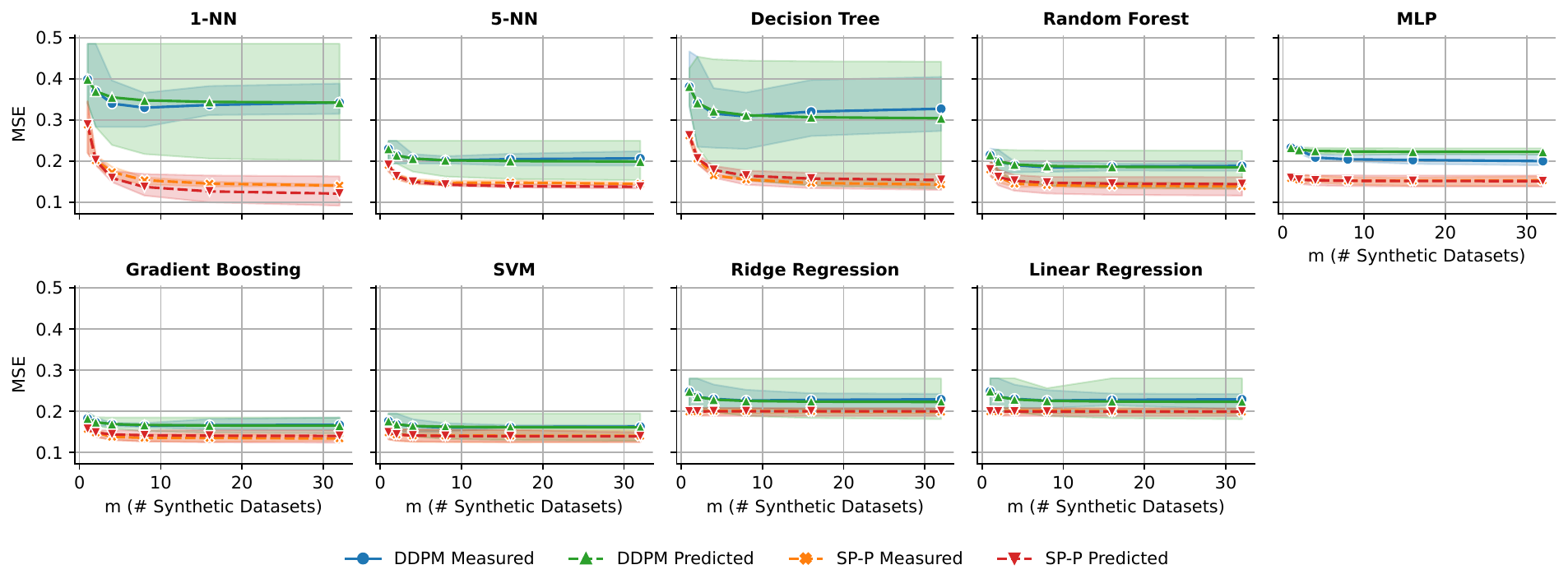}
        \caption{Insurance}
    \end{subfigure}
    \caption{
        MSE prediction on the last two regression datasets.
        The predictions for synthpop are very accurate, and
        the predictions for DDPM are accurate for most cases.
        1-NN and 5-NN are nearest neighbours with 1 or 5 neighbours. The results are averaged over 3 
        repeats with different train-test splits. The error bands are 95\% confidence 
        intervals formed by bootstrapping over the repeats. 
    }
    \label{fig:mse-prediction-regression2}
\end{figure*}

\begin{figure*}
    \begin{subfigure}{\textwidth}
        \centering
        \includegraphics[width=\textwidth]{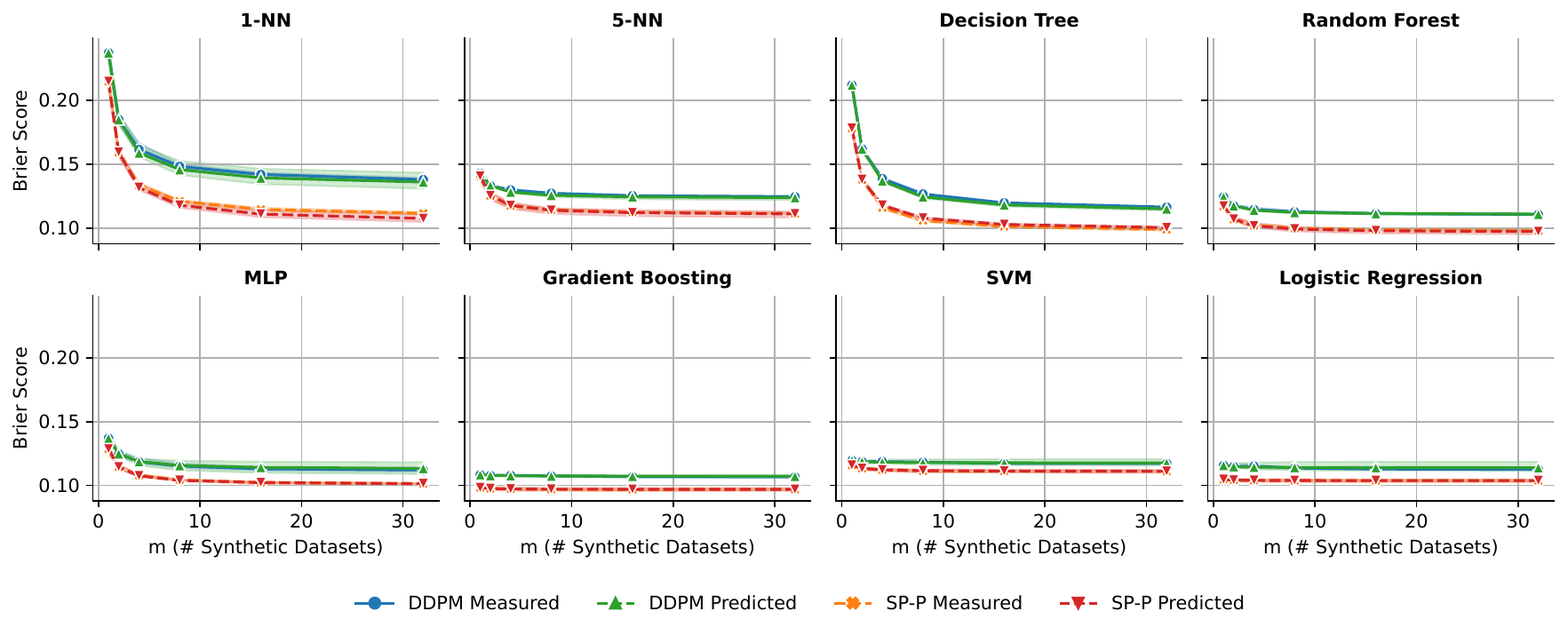}
        \caption{Adult}
    \end{subfigure}
    \begin{subfigure}{\textwidth}
        \centering
        \includegraphics[width=\textwidth]{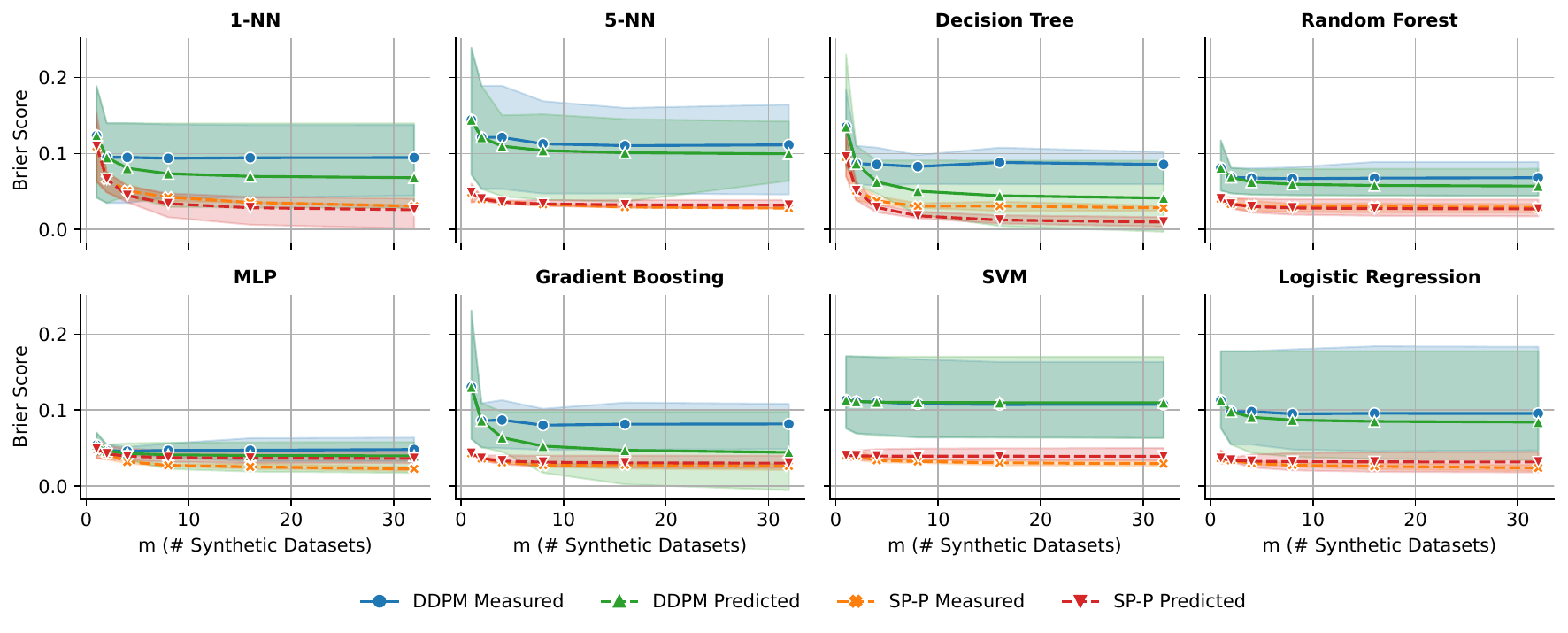}
        \caption{Breast Cancer}
    \end{subfigure}
    \begin{subfigure}{\textwidth}
        \centering
        \includegraphics[width=0.85\textwidth]{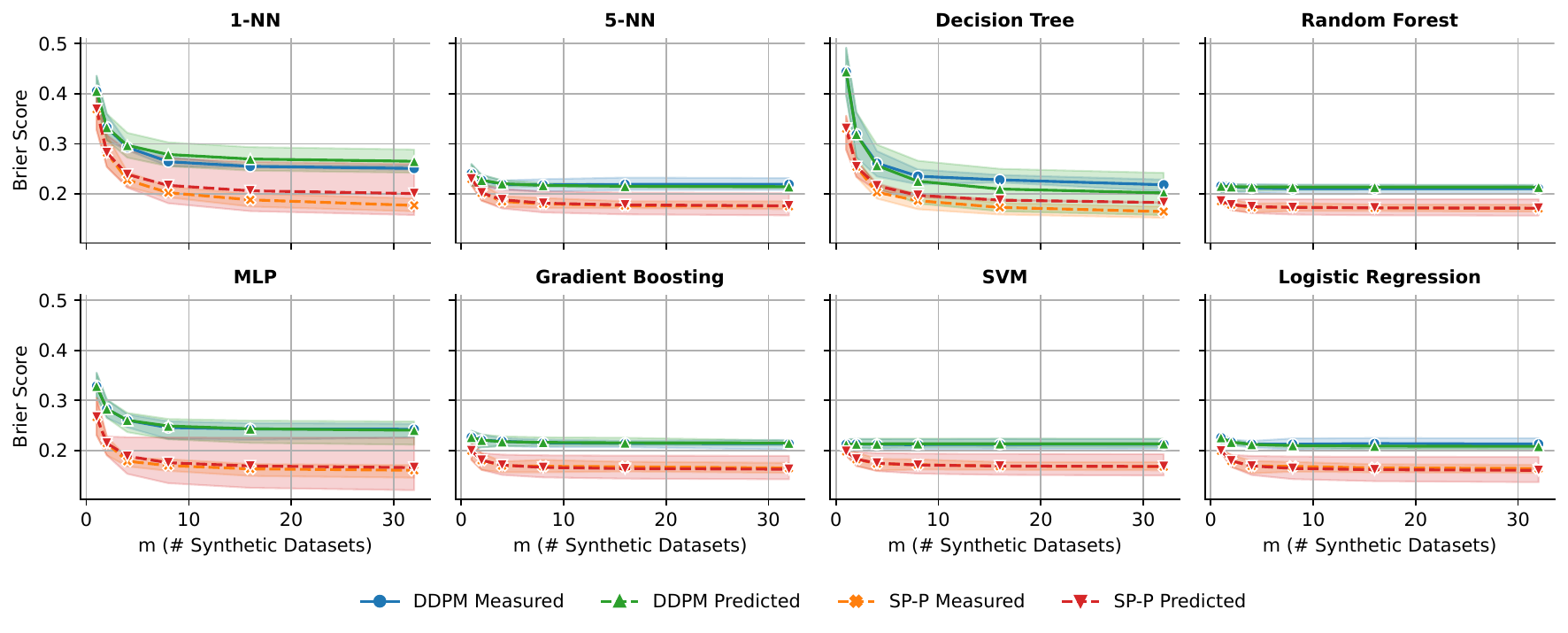}

        \vspace{-3mm}
        \caption{German credit}
    \end{subfigure}
    \caption{
        Brier score prediction on three classification datasets.
        The predictions are accurate, but can have high variance.
        1-NN and 5-NN are nearest neighbours with 1 or 5 neighbours. The results are averaged over 3 
        repeats with different train-test splits. The error bands are 95\% confidence 
        intervals formed by bootstrapping over the repeats. 
    }
    \label{fig:mse-prediction-classification1}
\end{figure*}

\begin{figure*}
    \begin{subfigure}{\textwidth}
        \centering
        \includegraphics[width=\textwidth]{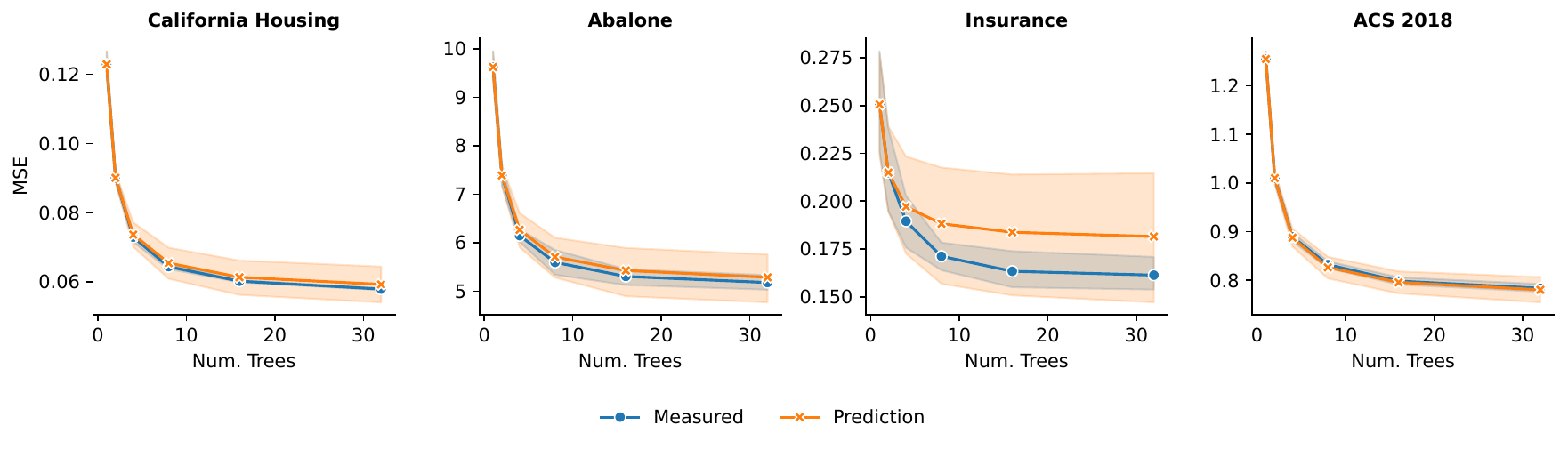}
        \caption{Regression}
        \label{fig:random-forest-mse-prediction-regression}
    \end{subfigure}
    \begin{subfigure}{\textwidth}
        \centering
        \includegraphics[width=0.8\textwidth]{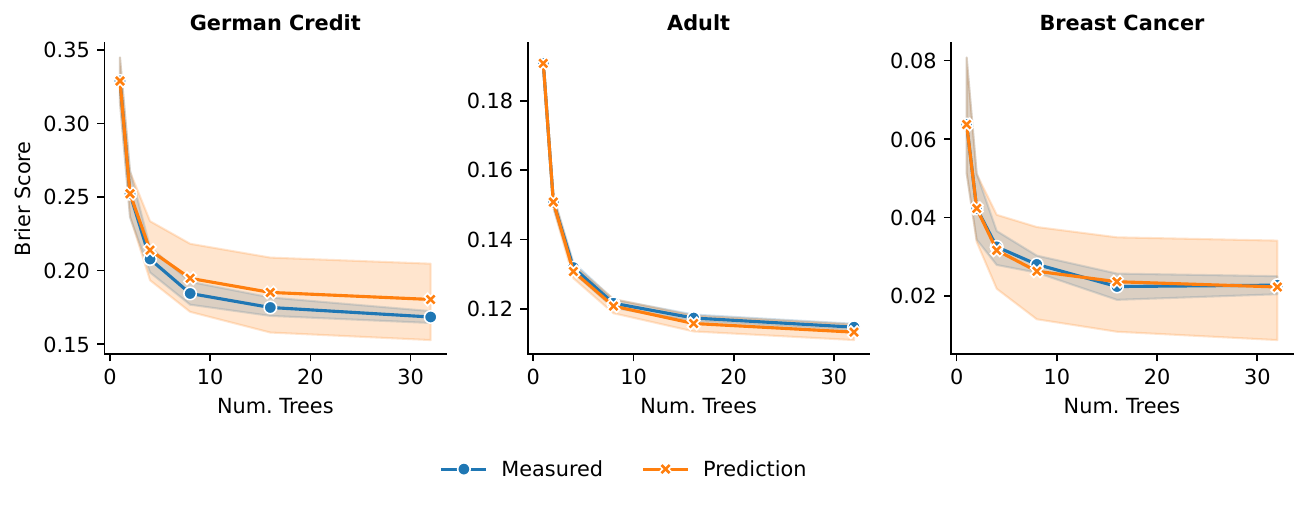}
        \caption{Classification}
        \label{fig:random-forest-mse-prediction-classification}
    \end{subfigure}
    \caption{
        Random forest performance prediction on the regression datasets in (a) and classification
        datasets in (b). The prediction is reasonably accurate on the datasets with accurate 
        estimates of the error. On the other datasets, the prediction can 
        have high variance. The lines show averages over 3 different train-test splits and 
        3 repeats of model training per split.
        The error bands are 95\% confidence intervals formed by bootsrapping over the 
        repeats and different splits.
    }
    \label{fig:random-forest-mse-prediction}
\end{figure*}

\begin{figure*}
    \centering
    \includegraphics[width=\textwidth]{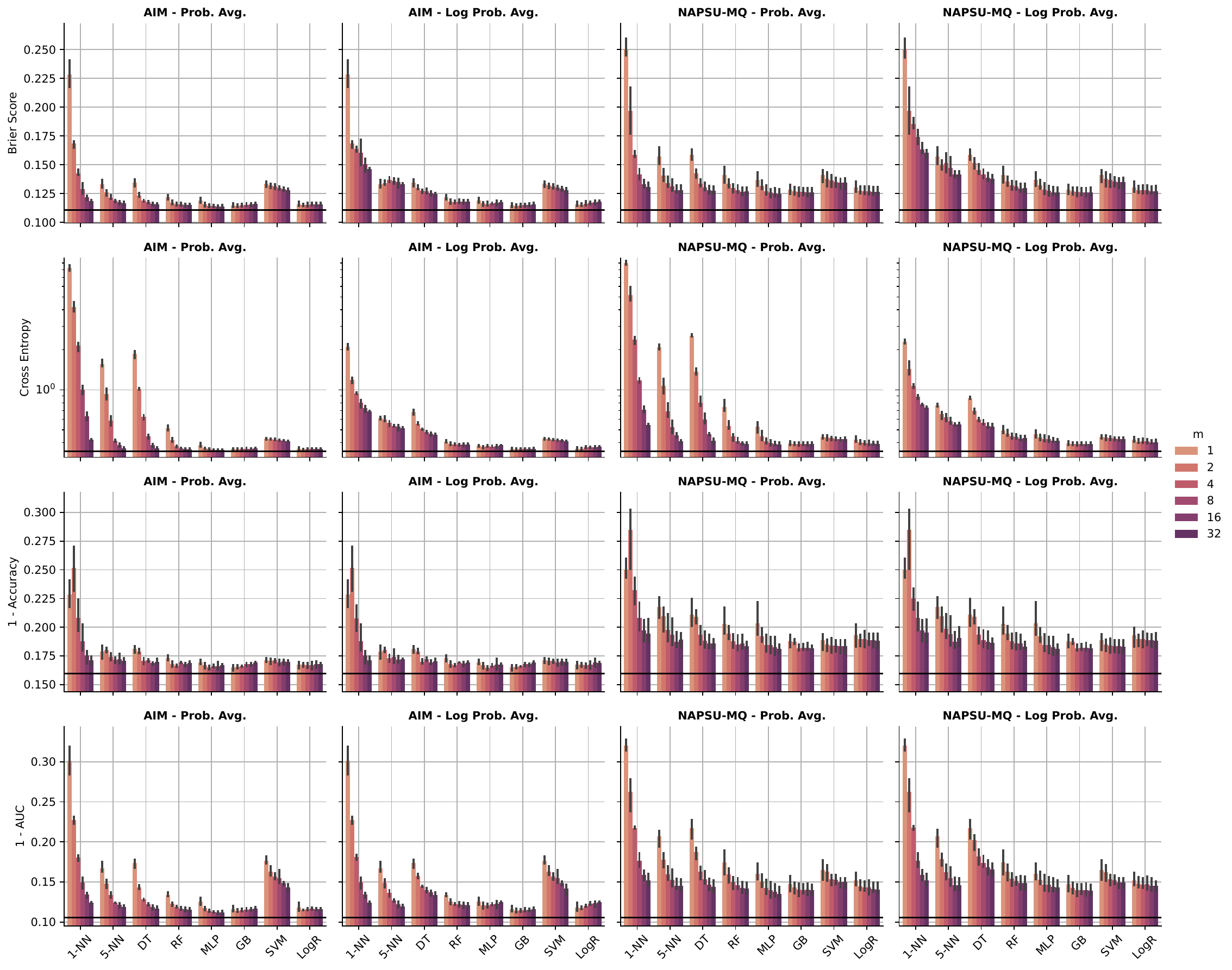}
    \caption{
        All error metrics on the Adult dataset with reduced features and DP synthetic data generation. 
        Note the he logarithmic scale on the 
        cross entropy y-axis. 
        The privacy parameters are $\epsilon = 1.5$, $\delta = n^{-2} \adultdelta$.
        The predictors are nearest neighbours 
        with 1 or 5 neighbours (1-NN and 5-NN), decision tree (DT), random forest (RF), a multilayer perceptron (MLP), 
        gradient boosted trees (GB), a support vector machine (SVM) and logistic regression (LogR).
        The black line show the loss of the best 
        downstream predictor trained on real data. The results are averaged over 3 repeats with 
        different train-test splits. The error bars are 95\% confidence intervals formed by 
        bootstrapping over the repeats.
    }
    \label{fig:dp-experiment-all-metrics}
\end{figure*}

\FloatBarrier
\newpage

\subsection{Result Tables}

\begin{table}
    \caption{
        Table of synthetic data generator comparison results from 
        Figure~\ref{fig:synthetic-data-algo-comparison}. 
        The numbers are 
        the mean MSE $\pm$ standard deviation from 3 repeats.
    }
    \label{table:synthetic-data-algo-comparison}
    \small
    \centering
\begin{tabular}{llllll}
\toprule
 & m & 1 & 2 & 5 & 10 \\
Downstream & Generator &  &  &  &  \\
\midrule
\multirow[c]{6}{*}{1-NN} & CTGAN & 0.2621 $\pm$ 0.0228 & 0.1775 $\pm$ 0.0037 & 0.1316 $\pm$ 0.0044 & 0.1164 $\pm$ 0.0040 \\
 & DDPM & 0.1365 $\pm$ 0.0029 & 0.1365 $\pm$ 0.0029 & 0.1365 $\pm$ 0.0029 & 0.1343 $\pm$ 0.0019 \\
 & DDPM-KL & 0.2447 $\pm$ 0.0082 & 0.2447 $\pm$ 0.0082 & 0.2238 $\pm$ 0.0240 & 0.2124 $\pm$ 0.0085 \\
 & SP-IP & 0.1544 $\pm$ 0.0040 & 0.1139 $\pm$ 0.0028 & 0.0892 $\pm$ 0.0018 & 0.0805 $\pm$ 0.0019 \\
 & SP-P & 0.1654 $\pm$ 0.0014 & 0.1172 $\pm$ 0.0012 & 0.0893 $\pm$ 0.0027 & 0.0801 $\pm$ 0.0015 \\
 & TVAE & 0.2195 $\pm$ 0.0036 & 0.2002 $\pm$ 0.0316 & 0.2114 $\pm$ 0.0123 & 0.1829 $\pm$ 0.0073 \\
\bottomrule
\end{tabular}

\end{table}

\begin{table}
    \caption{
        Table of results from the Abalone dataset. The numbers are 
        the mean MSE $\pm$ standard deviation from 3 repeats.
    }
    \label{table:abalone-results}
    \scriptsize
    \centering
    \begin{tabular}{llllllll}
\toprule
 & m & 1 & 2 & 4 & 8 & 16 & 32 \\
Downstream & Generator &  &  &  &  &  &  \\
\midrule
\multirow[c]{2}{*}{Linear Regression} & DDPM & 8.24 $\pm$ 2.392 & 8.24 $\pm$ 2.392 & 8.31 $\pm$ 2.334 & 8.25 $\pm$ 2.388 & 7.79 $\pm$ 1.602 & 7.79 $\pm$ 1.602 \\
 & SP-P & 5.57 $\pm$ 0.646 & 5.42 $\pm$ 0.429 & 5.35 $\pm$ 0.443 & 5.38 $\pm$ 0.470 & 5.35 $\pm$ 0.463 & 5.33 $\pm$ 0.452 \\
\cline{1-8}
\multirow[c]{2}{*}{Ridge Regression} & DDPM & 8.24 $\pm$ 2.392 & 8.24 $\pm$ 2.392 & 8.31 $\pm$ 2.335 & 8.25 $\pm$ 2.388 & 7.79 $\pm$ 1.599 & 7.79 $\pm$ 1.599 \\
 & SP-P & 5.57 $\pm$ 0.648 & 5.43 $\pm$ 0.432 & 5.35 $\pm$ 0.445 & 5.39 $\pm$ 0.472 & 5.35 $\pm$ 0.465 & 5.33 $\pm$ 0.454 \\
\cline{1-8}
\multirow[c]{2}{*}{1-NN} & DDPM & 9.01 $\pm$ 0.621 & 9.01 $\pm$ 0.621 & 8.57 $\pm$ 0.621 & 8.27 $\pm$ 0.123 & 7.72 $\pm$ 0.434 & 7.72 $\pm$ 0.434 \\
 & SP-P & 9.52 $\pm$ 0.890 & 7.35 $\pm$ 0.764 & 6.21 $\pm$ 0.482 & 5.66 $\pm$ 0.514 & 5.30 $\pm$ 0.439 & 5.13 $\pm$ 0.479 \\
\bottomrule
\end{tabular}

\end{table}

\begin{table}
    \caption{
        Table of results from the ACS 2018 dataset. The numbers are 
        the mean MSE $\pm$ standard deviation from 3 repeats. The nans 
        for linear regression are caused by excluding repeats with extremely large MSE ($\geq 10^6$).
    }
    \label{table:acs2018-results}
    \scriptsize
    \centering
    \begin{tabular}{llllllll}
\toprule
 & m & 1 & 2 & 4 & 8 & 16 & 32 \\
Downstream & Generator &  &  &  &  &  &  \\
\midrule
\multirow[c]{2}{*}{Linear Regression} & DDPM & 0.87 $\pm$ 0.008 & 0.87 $\pm$ 0.007 & 0.87 $\pm$ 0.005 & 0.87 $\pm$ 0.001 & 0.87 $\pm$ 0.003 & 0.87 $\pm$ 0.003 \\
 & SP-P & 0.78 $\pm$ 0.007 & 0.78 $\pm$ 0.006 & 0.78 $\pm$ 0.007 & 0.78 $\pm$ nan & nan & nan \\
\cline{1-8}
\multirow[c]{2}{*}{Ridge Regression} & DDPM & 0.87 $\pm$ 0.008 & 0.87 $\pm$ 0.007 & 0.87 $\pm$ 0.005 & 0.87 $\pm$ 0.003 & 0.87 $\pm$ 0.004 & 0.87 $\pm$ 0.003 \\
 & SP-P & 0.78 $\pm$ 0.007 & 0.78 $\pm$ 0.006 & 0.78 $\pm$ 0.007 & 0.78 $\pm$ 0.007 & 0.77 $\pm$ 0.007 & 0.77 $\pm$ 0.006 \\
\cline{1-8}
\multirow[c]{2}{*}{1-NN} & DDPM & 1.55 $\pm$ 0.022 & 1.26 $\pm$ 0.026 & 1.10 $\pm$ 0.019 & 1.03 $\pm$ 0.014 & 1.00 $\pm$ 0.013 & 0.98 $\pm$ 0.011 \\
 & SP-P & 1.38 $\pm$ 0.015 & 1.04 $\pm$ 0.010 & 0.88 $\pm$ 0.001 & 0.79 $\pm$ 0.004 & 0.75 $\pm$ 0.006 & 0.73 $\pm$ 0.007 \\
\bottomrule
\end{tabular}

\end{table}

\begin{table}
    \caption{
        Table of results from the California Housing dataset. The numbers are 
        the mean MSE $\pm$ standard deviation from 3 repeats.
    }
    \label{table:california-housing-results}
    \scriptsize
    \centering
    \begin{tabular}{llllllll}
\toprule
 & m & 1 & 2 & 4 & 8 & 16 & 32 \\
Downstream & Generator &  &  &  &  &  &  \\
\midrule
\multirow[c]{2}{*}{Linear Regression} & DDPM & 0.12 $\pm$ 0.004 & 0.12 $\pm$ 0.004 & 0.12 $\pm$ 0.003 & 0.12 $\pm$ 0.003 & 0.12 $\pm$ 0.003 & 0.12 $\pm$ 0.003 \\
 & SP-P & 0.12 $\pm$ 0.001 & 0.11 $\pm$ 0.002 & 0.11 $\pm$ 0.002 & 0.11 $\pm$ 0.001 & 0.11 $\pm$ 0.002 & 0.11 $\pm$ 0.002 \\
\cline{1-8}
\multirow[c]{2}{*}{Ridge Regression} & DDPM & 0.12 $\pm$ 0.004 & 0.12 $\pm$ 0.004 & 0.12 $\pm$ 0.003 & 0.12 $\pm$ 0.003 & 0.12 $\pm$ 0.003 & 0.12 $\pm$ 0.003 \\
 & SP-P & 0.12 $\pm$ 0.001 & 0.11 $\pm$ 0.002 & 0.11 $\pm$ 0.002 & 0.11 $\pm$ 0.001 & 0.11 $\pm$ 0.002 & 0.11 $\pm$ 0.002 \\
\cline{1-8}
\multirow[c]{2}{*}{1-NN} & DDPM & 0.14 $\pm$ 0.008 & 0.14 $\pm$ 0.008 & 0.12 $\pm$ 0.006 & 0.13 $\pm$ 0.007 & 0.12 $\pm$ 0.004 & 0.12 $\pm$ 0.004 \\
 & SP-P & 0.16 $\pm$ 0.002 & 0.12 $\pm$ 0.002 & 0.09 $\pm$ 0.002 & 0.08 $\pm$ 0.001 & 0.08 $\pm$ 0.002 & 0.08 $\pm$ 0.001 \\
\bottomrule
\end{tabular}

\end{table}

\begin{table}
    \caption{
        Table of results from the Insurance dataset. The numbers are 
        the mean MSE $\pm$ standard deviation from 3 repeats.
    }
    \label{table:insurance-results}
    \scriptsize
    \centering
    \begin{tabular}{llllllll}
\toprule
 & m & 1 & 2 & 4 & 8 & 16 & 32 \\
Downstream & Generator &  &  &  &  &  &  \\
\midrule
\multirow[c]{2}{*}{Linear Regression} & DDPM & 0.25 $\pm$ 0.037 & 0.24 $\pm$ 0.039 & 0.23 $\pm$ 0.030 & 0.23 $\pm$ 0.023 & 0.23 $\pm$ 0.020 & 0.23 $\pm$ 0.021 \\
 & SP-P & 0.20 $\pm$ 0.009 & 0.20 $\pm$ 0.009 & 0.20 $\pm$ 0.008 & 0.20 $\pm$ 0.006 & 0.20 $\pm$ 0.006 & 0.20 $\pm$ 0.006 \\
\cline{1-8}
\multirow[c]{2}{*}{Ridge Regression} & DDPM & 0.25 $\pm$ 0.037 & 0.24 $\pm$ 0.039 & 0.23 $\pm$ 0.031 & 0.23 $\pm$ 0.023 & 0.23 $\pm$ 0.021 & 0.23 $\pm$ 0.021 \\
 & SP-P & 0.20 $\pm$ 0.009 & 0.20 $\pm$ 0.010 & 0.20 $\pm$ 0.008 & 0.20 $\pm$ 0.007 & 0.20 $\pm$ 0.006 & 0.20 $\pm$ 0.006 \\
\cline{1-8}
\multirow[c]{2}{*}{1-NN} & DDPM & 0.40 $\pm$ 0.076 & 0.37 $\pm$ 0.104 & 0.34 $\pm$ 0.056 & 0.33 $\pm$ 0.042 & 0.34 $\pm$ 0.040 & 0.34 $\pm$ 0.040 \\
 & SP-P & 0.29 $\pm$ 0.063 & 0.20 $\pm$ 0.012 & 0.17 $\pm$ 0.015 & 0.15 $\pm$ 0.009 & 0.15 $\pm$ 0.006 & 0.14 $\pm$ 0.002 \\
\bottomrule
\end{tabular}

\end{table}

\begin{table}
    \caption{
        Table of results from the Adult dataset. The numbers are 
        the mean Brier score $\pm$ standard deviation from 3 repeats.
    }
    \label{table:adult-brier-results}
    \tiny
    \centering
    \begin{tabular}{llllllll}
\toprule
 & m & 1 & 2 & 4 & 8 & 16 & 32 \\
Downstream & Generator &  &  &  &  &  &  \\
\midrule
\multirow[c]{4}{*}{1-NN} & DDPM - Log Prob. Avg. & 0.24 $\pm$ 0.005 & 0.19 $\pm$ 0.004 & 0.19 $\pm$ 0.006 & 0.18 $\pm$ 0.002 & 0.18 $\pm$ 0.003 & 0.17 $\pm$ 0.003 \\
 & DDPM - Prob. Avg. & 0.24 $\pm$ 0.005 & 0.19 $\pm$ 0.004 & 0.16 $\pm$ 0.005 & 0.15 $\pm$ 0.002 & 0.14 $\pm$ 0.002 & 0.14 $\pm$ 0.002 \\
 & SP-P - Log Prob. Avg. & 0.21 $\pm$ 0.005 & 0.16 $\pm$ 0.003 & 0.15 $\pm$ 0.003 & 0.15 $\pm$ 0.005 & 0.14 $\pm$ 0.003 & 0.14 $\pm$ 0.003 \\
 & SP-P - Prob. Avg. & 0.21 $\pm$ 0.005 & 0.16 $\pm$ 0.003 & 0.13 $\pm$ 0.002 & 0.12 $\pm$ 0.002 & 0.11 $\pm$ 0.001 & 0.11 $\pm$ 0.001 \\
\cline{1-8}
\multirow[c]{4}{*}{5-NN} & DDPM - Log Prob. Avg. & 0.14 $\pm$ 0.003 & 0.14 $\pm$ 0.001 & 0.14 $\pm$ 0.002 & 0.14 $\pm$ 0.002 & 0.13 $\pm$ 0.001 & 0.13 $\pm$ 0.000 \\
 & DDPM - Prob. Avg. & 0.14 $\pm$ 0.003 & 0.13 $\pm$ 0.001 & 0.13 $\pm$ 0.001 & 0.13 $\pm$ 0.001 & 0.13 $\pm$ 0.000 & 0.12 $\pm$ 0.000 \\
\bottomrule
\end{tabular}

\end{table}

\begin{table}
    \caption{
        Table of results from the Breast Cancer dataset. The numbers are 
        the mean Brier score $\pm$ standard deviation from 3 repeats.
    }
    \label{table:breast-cancer-brier-results}
    \tiny
    \centering
    \begin{tabular}{llllllll}
\toprule
 & m & 1 & 2 & 4 & 8 & 16 & 32 \\
Downstream & Generator &  &  &  &  &  &  \\
\midrule
\multirow[c]{4}{*}{1-NN} & DDPM - Log Prob. Avg. & 0.12 $\pm$ 0.075 & 0.09 $\pm$ 0.054 & 0.09 $\pm$ 0.054 & 0.10 $\pm$ 0.043 & 0.10 $\pm$ 0.043 & 0.10 $\pm$ 0.043 \\
 & DDPM - Prob. Avg. & 0.12 $\pm$ 0.075 & 0.09 $\pm$ 0.054 & 0.09 $\pm$ 0.054 & 0.09 $\pm$ 0.049 & 0.09 $\pm$ 0.048 & 0.09 $\pm$ 0.047 \\
 & SP-P - Log Prob. Avg. & 0.11 $\pm$ 0.045 & 0.07 $\pm$ 0.015 & 0.05 $\pm$ 0.013 & 0.04 $\pm$ 0.002 & 0.04 $\pm$ 0.003 & 0.03 $\pm$ 0.003 \\
 & SP-P - Prob. Avg. & 0.11 $\pm$ 0.045 & 0.07 $\pm$ 0.015 & 0.05 $\pm$ 0.011 & 0.04 $\pm$ 0.004 & 0.04 $\pm$ 0.000 & 0.03 $\pm$ 0.001 \\
\cline{1-8}
\multirow[c]{4}{*}{5-NN} & DDPM - Log Prob. Avg. & 0.14 $\pm$ 0.086 & 0.13 $\pm$ 0.087 & 0.13 $\pm$ 0.087 & 0.13 $\pm$ 0.081 & 0.12 $\pm$ 0.070 & 0.12 $\pm$ 0.077 \\
 & DDPM - Prob. Avg. & 0.14 $\pm$ 0.086 & 0.12 $\pm$ 0.068 & 0.12 $\pm$ 0.068 & 0.11 $\pm$ 0.061 & 0.11 $\pm$ 0.058 & 0.11 $\pm$ 0.060 \\
\bottomrule
\end{tabular}

\end{table}

\begin{table}
    \caption{
        Table of results from the German Credit dataset. The numbers are 
        the mean Brier score $\pm$ standard deviation from 3 repeats.
    }
    \label{table:german-credit-brier-results}
    \tiny
    \centering
    \begin{tabular}{llllllll}
\toprule
 & m & 1 & 2 & 4 & 8 & 16 & 32 \\
Downstream & Generator &  &  &  &  &  &  \\
\midrule
\multirow[c]{4}{*}{1-NN} & DDPM - Log Prob. Avg. & 0.41 $\pm$ 0.030 & 0.33 $\pm$ 0.027 & 0.35 $\pm$ 0.017 & 0.34 $\pm$ 0.021 & 0.32 $\pm$ 0.014 & 0.32 $\pm$ 0.018 \\
 & DDPM - Prob. Avg. & 0.41 $\pm$ 0.030 & 0.33 $\pm$ 0.027 & 0.29 $\pm$ 0.007 & 0.26 $\pm$ 0.010 & 0.25 $\pm$ 0.002 & 0.25 $\pm$ 0.005 \\
 & SP-P - Log Prob. Avg. & 0.37 $\pm$ 0.037 & 0.28 $\pm$ 0.033 & 0.27 $\pm$ 0.026 & 0.26 $\pm$ 0.024 & 0.24 $\pm$ 0.031 & 0.23 $\pm$ 0.028 \\
 & SP-P - Prob. Avg. & 0.37 $\pm$ 0.037 & 0.28 $\pm$ 0.033 & 0.23 $\pm$ 0.015 & 0.20 $\pm$ 0.011 & 0.19 $\pm$ 0.013 & 0.18 $\pm$ 0.013 \\
\cline{1-8}
\multirow[c]{4}{*}{5-NN} & DDPM - Log Prob. Avg. & 0.24 $\pm$ 0.022 & 0.25 $\pm$ 0.013 & 0.24 $\pm$ 0.019 & 0.25 $\pm$ 0.025 & 0.25 $\pm$ 0.028 & 0.25 $\pm$ 0.029 \\
 & DDPM - Prob. Avg. & 0.24 $\pm$ 0.022 & 0.23 $\pm$ 0.011 & 0.22 $\pm$ 0.010 & 0.22 $\pm$ 0.012 & 0.22 $\pm$ 0.013 & 0.22 $\pm$ 0.012 \\
\bottomrule
\end{tabular}

\end{table}

\begin{table}
    \caption{
        Table of results from predicting MSE on the ACS 2018 dataset. 
        The numbers are the mean measured and predicted MSEs $\pm$ standard 
        deviations from 3 repeats. The nans for linear regression are due to 
        excluding some repeats that had an extremely large MSE ($\geq 10^6$).
    }
    \label{table:mse-prediction-acs-2018}
    \tiny
    \centering
    \begin{tabular}{lllllllll}
\toprule
 &  & m & 1 & 2 & 4 & 8 & 16 & 32 \\
Downstream & Generator & &  &  &  &  &  &  \\
\midrule
\multirow[c]{4}{*}{Linear Regression} & \multirow[c]{2}{*}{DDPM} & Predicted & 0.87 $\pm$ 0.008 & 0.87 $\pm$ 0.007 & 0.87 $\pm$ 0.008 & 0.88 $\pm$ 0.006 & 0.88 $\pm$ 0.007 & 0.88 $\pm$ 0.008 \\
 &  & Measured & 0.87 $\pm$ 0.008 & 0.87 $\pm$ 0.007 & 0.87 $\pm$ 0.005 & 0.87 $\pm$ 0.001 & 0.87 $\pm$ 0.003 & 0.87 $\pm$ 0.003 \\
\cline{2-9}
 & \multirow[c]{2}{*}{SP-P} & Predicted & 0.78 $\pm$ 0.007 & 0.78 $\pm$ 0.006 & 0.78 $\pm$ 0.006 & 0.78 $\pm$ nan & nan & nan \\
 &  & Measured & 0.78 $\pm$ 0.007 & 0.78 $\pm$ 0.006 & 0.78 $\pm$ 0.007 & 0.78 $\pm$ nan & nan & nan \\
\cline{1-9} \cline{2-9}
\multirow[c]{4}{*}{Ridge Regression} & \multirow[c]{2}{*}{DDPM} & Predicted & 0.87 $\pm$ 0.008 & 0.87 $\pm$ 0.007 & 0.87 $\pm$ 0.008 & 0.87 $\pm$ 0.009 & 0.87 $\pm$ 0.010 & 0.87 $\pm$ 0.010 \\
 &  & Measured & 0.87 $\pm$ 0.008 & 0.87 $\pm$ 0.007 & 0.87 $\pm$ 0.005 & 0.87 $\pm$ 0.003 & 0.87 $\pm$ 0.004 & 0.87 $\pm$ 0.003 \\
\bottomrule
\end{tabular}

\end{table}

\begin{table}
    \caption{
        Table of results from predicting Brier score on the German Credit dataset. 
        The numbers are the mean measured and predicted Brier scores $\pm$ standard 
        deviations from 3 repeats. 
    }
    \label{table:mse-prediction-german-credit}
    \tiny
    \centering
    \begin{tabular}{lllllllll}
\toprule
 &  & m & 1 & 2 & 4 & 8 & 16 & 32 \\
Downstream & Generator & &  &  &  &  &  &  \\
\midrule
\multirow[c]{4}{*}{1-NN} & \multirow[c]{2}{*}{DDPM} & Measured & 0.41 $\pm$ 0.027 & 0.33 $\pm$ 0.024 & 0.32 $\pm$ 0.032 & 0.30 $\pm$ 0.044 & 0.29 $\pm$ 0.036 & 0.28 $\pm$ 0.038 \\
 &  & Predicted & 0.41 $\pm$ 0.027 & 0.33 $\pm$ 0.024 & 0.30 $\pm$ 0.022 & 0.28 $\pm$ 0.021 & 0.27 $\pm$ 0.021 & 0.27 $\pm$ 0.021 \\
\cline{2-9}
 & \multirow[c]{2}{*}{SP-P} & Measured & 0.37 $\pm$ 0.033 & 0.28 $\pm$ 0.029 & 0.25 $\pm$ 0.030 & 0.23 $\pm$ 0.035 & 0.21 $\pm$ 0.035 & 0.20 $\pm$ 0.033 \\
 &  & Predicted & 0.37 $\pm$ 0.033 & 0.28 $\pm$ 0.029 & 0.24 $\pm$ 0.037 & 0.22 $\pm$ 0.043 & 0.21 $\pm$ 0.046 & 0.20 $\pm$ 0.047 \\
\cline{1-9} \cline{2-9}
\multirow[c]{4}{*}{5-NN} & \multirow[c]{2}{*}{DDPM} & Measured & 0.24 $\pm$ 0.020 & 0.24 $\pm$ 0.015 & 0.23 $\pm$ 0.019 & 0.23 $\pm$ 0.024 & 0.23 $\pm$ 0.025 & 0.23 $\pm$ 0.025 \\
 &  & Predicted & 0.24 $\pm$ 0.020 & 0.24 $\pm$ 0.015 & 0.23 $\pm$ 0.017 & 0.23 $\pm$ 0.020 & 0.23 $\pm$ 0.021 & 0.23 $\pm$ 0.022 \\
\bottomrule
\end{tabular}

\end{table}

\begin{table}
    \caption{Table of results from the DP experiment. The numbers are the mean Brier score $\pm$
    the standard deviation from 3 repeats.
    The privacy parameters are $\epsilon = 1.5$, $\delta = n^{-2} \adultdelta$.
    }
    \label{table:dp-experiment-brier}
    \tiny
    \centering 
    \begin{tabular}{llllllll}
\toprule
 & m & 1 & 2 & 4 & 8 & 16 & 32 \\
Downstream & Generator &  &  &  &  &  &  \\
\midrule
\multirow[c]{4}{*}{1-NN} & AIM - Log Prob. Avg. & 0.23 $\pm$ 0.012 & 0.17 $\pm$ 0.003 & 0.16 $\pm$ 0.002 & 0.16 $\pm$ 0.011 & 0.15 $\pm$ 0.006 & 0.15 $\pm$ 0.001 \\
 & AIM - Prob. Avg. & 0.23 $\pm$ 0.012 & 0.17 $\pm$ 0.003 & 0.14 $\pm$ 0.003 & 0.13 $\pm$ 0.004 & 0.12 $\pm$ 0.002 & 0.12 $\pm$ 0.001 \\
 & NAPSU-MQ - Log Prob. Avg. & 0.25 $\pm$ 0.008 & 0.20 $\pm$ 0.020 & 0.19 $\pm$ 0.005 & 0.17 $\pm$ 0.006 & 0.16 $\pm$ 0.005 & 0.16 $\pm$ 0.004 \\
 & NAPSU-MQ - Prob. Avg. & 0.25 $\pm$ 0.008 & 0.20 $\pm$ 0.020 & 0.16 $\pm$ 0.003 & 0.14 $\pm$ 0.005 & 0.13 $\pm$ 0.003 & 0.13 $\pm$ 0.003 \\
\cline{1-8}
\multirow[c]{4}{*}{5-NN} & AIM - Log Prob. Avg. & 0.13 $\pm$ 0.004 & 0.13 $\pm$ 0.002 & 0.14 $\pm$ 0.002 & 0.14 $\pm$ 0.002 & 0.13 $\pm$ 0.004 & 0.13 $\pm$ 0.001 \\
 & AIM - Prob. Avg. & 0.13 $\pm$ 0.004 & 0.13 $\pm$ 0.003 & 0.12 $\pm$ 0.002 & 0.12 $\pm$ 0.001 & 0.12 $\pm$ 0.002 & 0.12 $\pm$ 0.001 \\
\bottomrule
\end{tabular}

\end{table}

\FloatBarrier

\subsection{Estimating Model and Synthetic Data Variances}\label{app:variance-estimation-experiment}
In this section, we estimate the MV and SDV terms from the decomposition in 
Theorem~\ref{thm:mse-synthetic-data-decomposition}. We first generate 
32 synthetic datasets that are 5 times larger than the real dataset, and split 
each synthetic datasets into 5 equally-sized subsets. This is equivalent to 
training 32 generators, with parameters $\theta_i$, and 
for each generator, generating 5 synthetic 
datasets i.i.d.\  We then train the downstream predictor on each synthetic dataset,
and store the predictions for all test points.

To estimate MV, we compute the sample variance over the 5 synthetic datasets
generated from the same $\theta_i$, and then compute the mean over the 32 different 
$\theta_i$ values. To estimate SDV, we compute the sample mean over the 5 synthetic 
datasets from the same $\theta_i$, and compute the sample variance over the 
32 different $\theta_i$ values.

The result is an estimate of MV and SDV for each test point. We plot the mean 
over the test points. The whole experiment is repeated 3 times, with different 
train-test splits. The datasets, train-test splits, and downstream predictors 
are the same as in the other experiments, described in 
Appendix~\ref{sec:experiment-details}.

The results are in Figure~\ref{fig:variance-estimation}. MV depends mostly 
on the downstream predictor, while SDV also depends on the synthetic data 
generator. We also confirm that decision trees and 1-NN have much higher variance 
than the other models, and linear, ridge and logistic regression have a 
very low variance.

\begin{figure*}
    \begin{subfigure}{0.5\textwidth}
        \centering 
        \includegraphics[width=\textwidth]{figures/variance-estimation/abalone.pdf}
        \vspace{-6mm}
        \caption{Abalone}
    \end{subfigure}
    \begin{subfigure}{0.5\textwidth}
        \centering 
        \includegraphics[width=\textwidth]{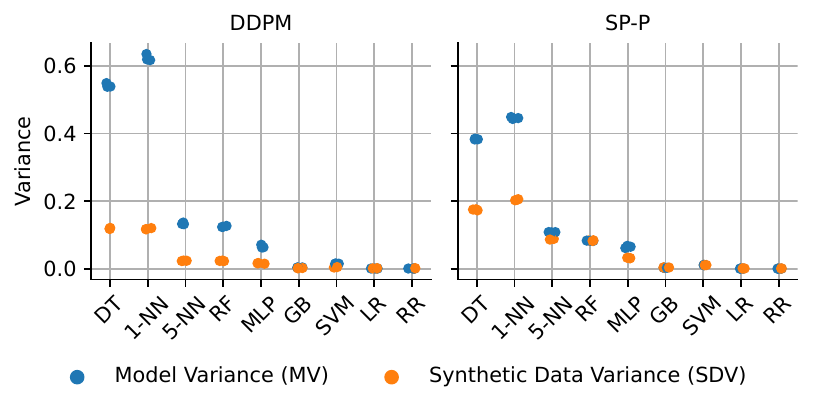}
        \vspace{-6mm}
        \caption{ACS 2018}
    \end{subfigure}

    \vspace{2mm} % The empty line is required to get the spacing in the correct place
    \begin{subfigure}{0.5\textwidth}
        \centering 
        \includegraphics[width=\textwidth]{figures/variance-estimation/adult.pdf}
        \vspace{-6mm}
        \caption{Adult}
    \end{subfigure}
    \begin{subfigure}{0.5\textwidth}
        \centering 
        \includegraphics[width=\textwidth]{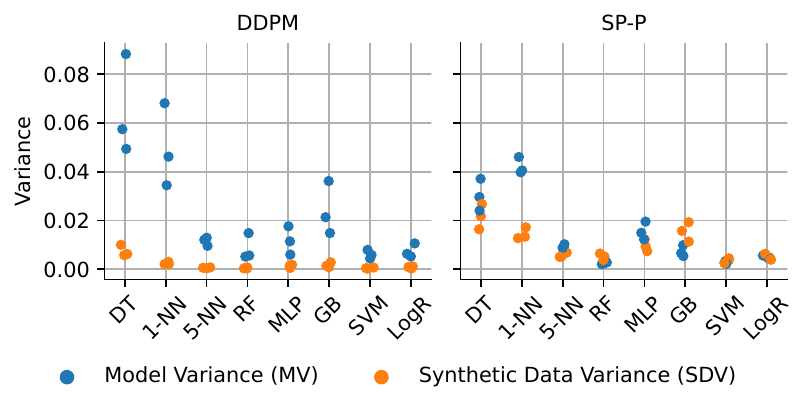}
        \vspace{-6mm}
        \caption{Breast Cancer}
    \end{subfigure}

    \vspace{2mm}
    \begin{subfigure}{0.5\textwidth}
        \centering 
        \includegraphics[width=\textwidth]{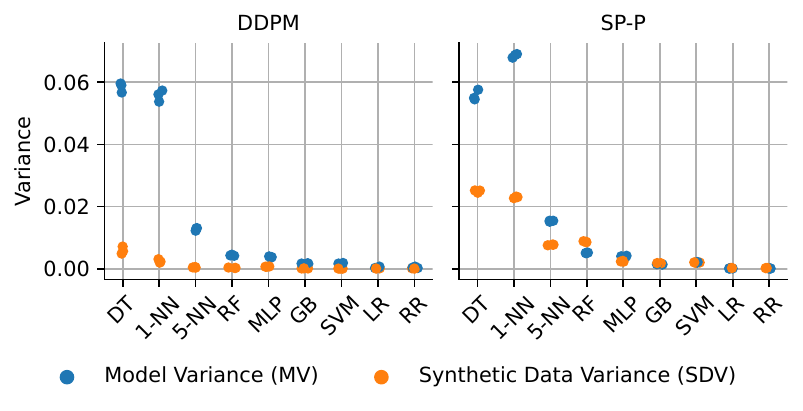}
        \vspace{-6mm}
        \caption{California Housing}
    \end{subfigure}
    \begin{subfigure}{0.5\textwidth}
        \centering 
        \includegraphics[width=\textwidth]{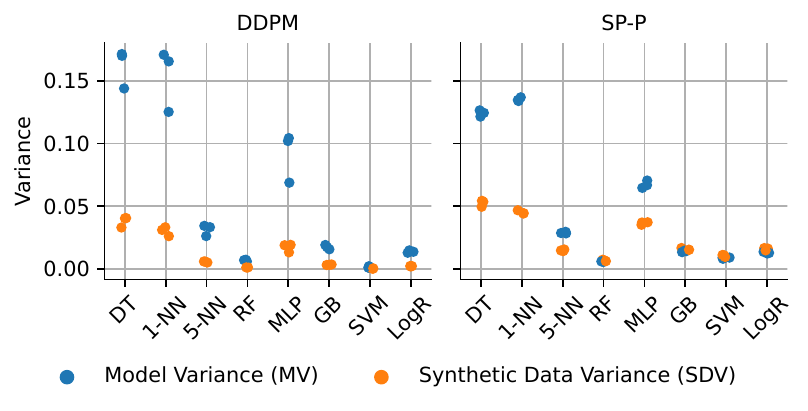}
        \vspace{-6mm}
        \caption{German Credit}
    \end{subfigure}

    \vspace{2mm}
    \begin{subfigure}{1.0\textwidth}
        \centering 
        \includegraphics[width=0.5\textwidth]{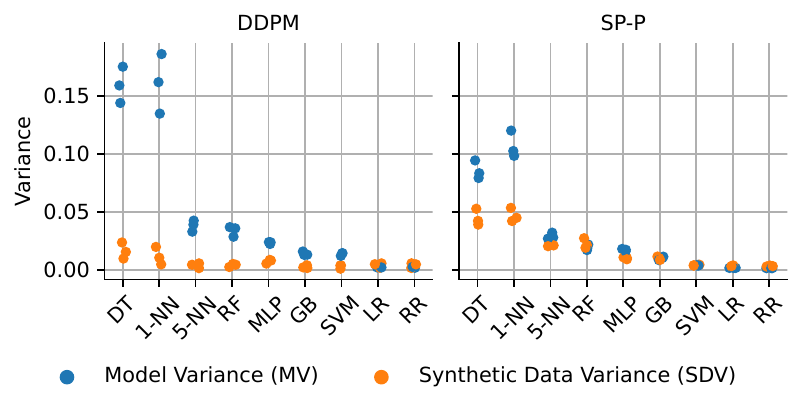}
        \vspace{-2mm}
        \caption{Insurance}
    \end{subfigure}
    \caption{
        Estimating the MV and SDV terms from the decomposition. Decision trees 
        have high variances on all datasets, while linear, ridge and logistic regression
        have low variances. MV depends mostly on the predictor, while SDV depends 
        on both the predictor and synthetic data generation algorithm. The 
        points are the averages of estimated MV and SDV,
        averaged over the test data, from 3 repeats with different train-test splits.
        We excluded some test data points that had extremely large variance estimates
        ($\geq 10^6$) for linear regression on the ACS 2018 dataset.
        The predictors are decision tree (DT), nearest neighbours with 1 or 5 neighbours (1-NN and 5-NN), 
        random forest (RF), a multilayer perceptron (MLP), 
        gradient boosted trees (GB), a support vector machine (SVM), 
        linear regression (LR), ridge regression (RR) and logistic regression (LogR).
        The synthetic data generators are DDPM and synthpop (SP-P).
    }
    \label{fig:variance-estimation}
\end{figure*}

\subsection{Comparison with Generating A Single Large Synthetic Dataset}\label{app:one-large-synthetic-dataset}
In this section, we examine an alternative to multiple synthetic datasets:
generating a single, large synthetic dataset. 
\citet{vanbreugelSyntheticDataReal2023} found this lead to poor model
evaluation and selection, with a single synthetic dataset leading to 
overestimating the performance of complex models. They did not 
directly compare the effect on accuracy, which is what we will do 
in this section.

We use the synthetic datasets from the variance estimation experiment\footnote{We use 
only one of the 32 synthetic datasets per repeat.} of 
Section~\ref{app:variance-estimation-experiment} that are 5 times larger than the real 
dataset. We also take smaller
subsets of the whole synthetic dataset to examine other synthetic 
dataset sizes. We train the same models as in the other experiments,
but only on regression datasets generated by synthpop (proper).

We compare the results from the single large synthetic dataset with 
the ensemble of multiple synthetic dataset with an equal number of 
total datapoints in Figure~\ref{fig:one-large-results-regression}. We 
see that the ensemble of multiple synthetic datasets is equal or better 
in all cases in terms of MSE. 

\begin{figure*}
    \begin{subfigure}{1.0\textwidth}
        \centering 
        \includegraphics[width=\textwidth]{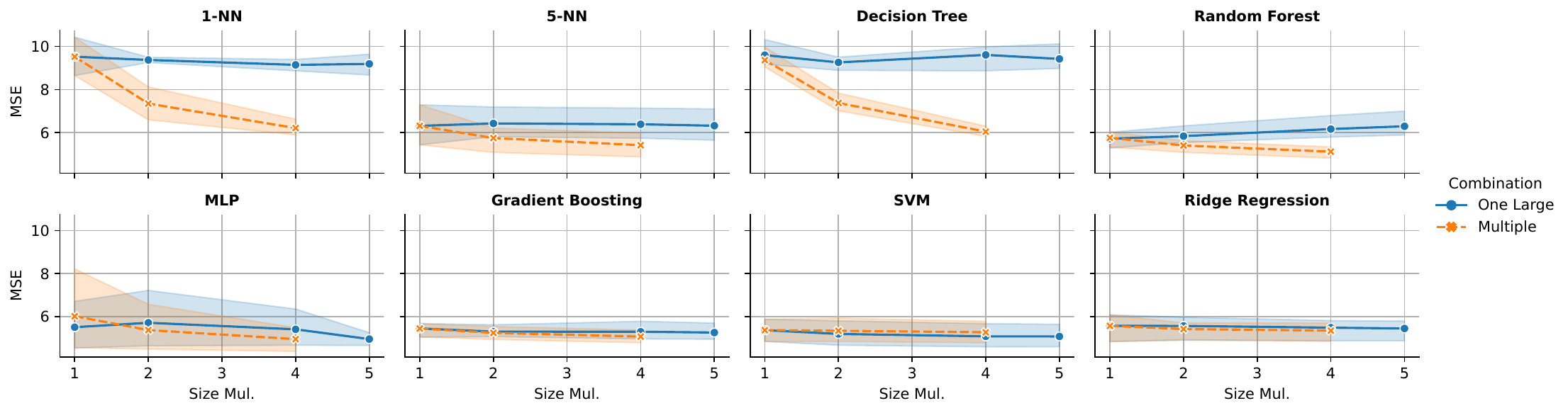}
        \caption{Abalone}
        \label{fig:one-large-regression-abalone}
    \end{subfigure}
    \begin{subfigure}{1.0\textwidth}
        \centering 
        \includegraphics[width=\textwidth]{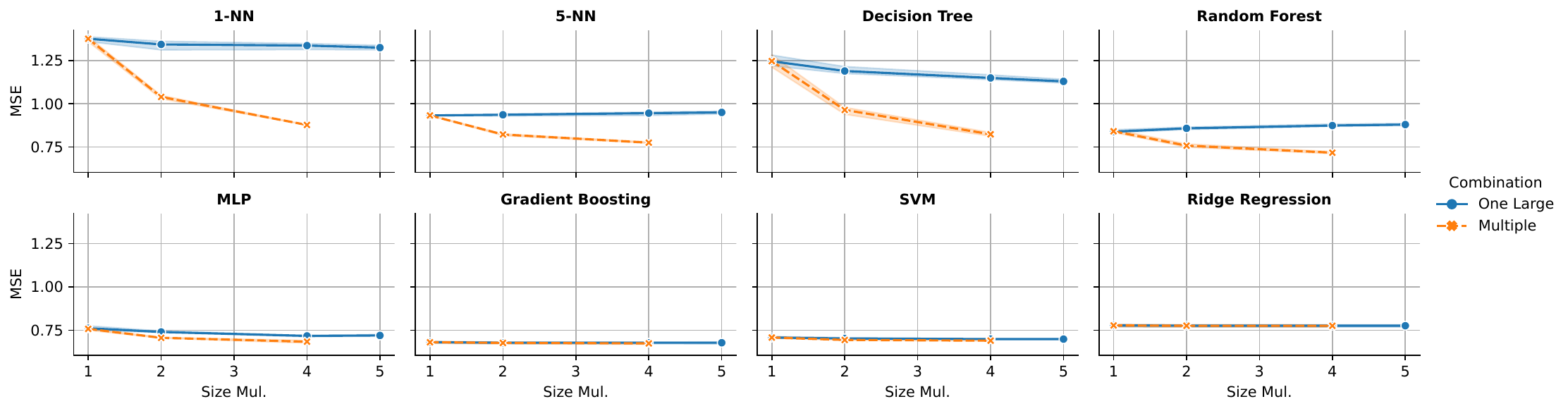}
        \caption{ACS 2018}
        \label{fig:one-large-regression-acs2018}
    \end{subfigure}
    \begin{subfigure}{1.0\textwidth}
        \centering 
        \includegraphics[width=\textwidth]{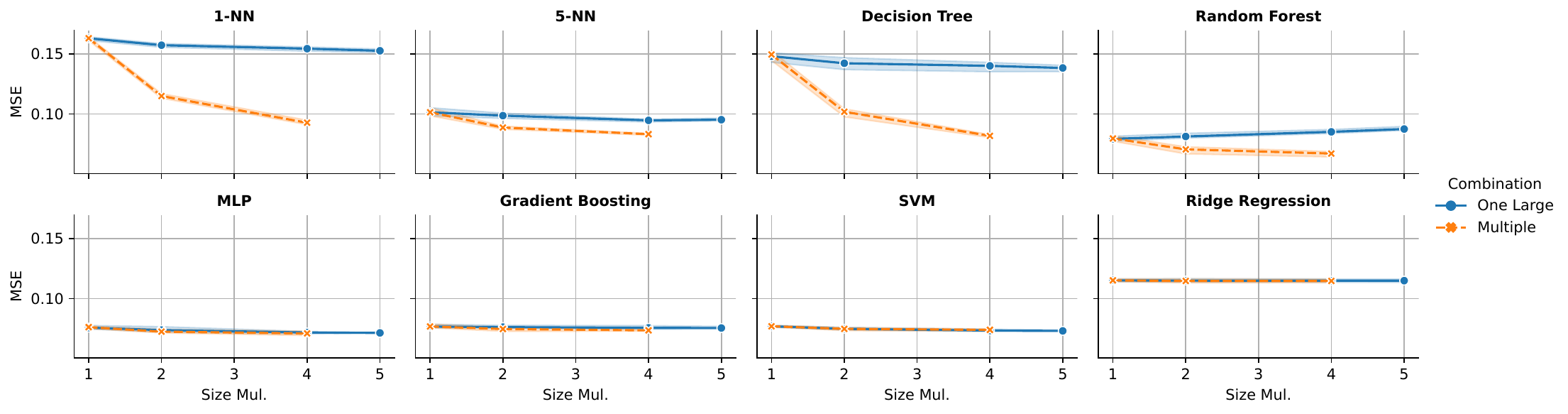}
        \caption{California Housing}
        \label{fig:one-large-regression-california-housing}
    \end{subfigure}
    \begin{subfigure}{1.0\textwidth}
        \centering 
        \includegraphics[width=\textwidth]{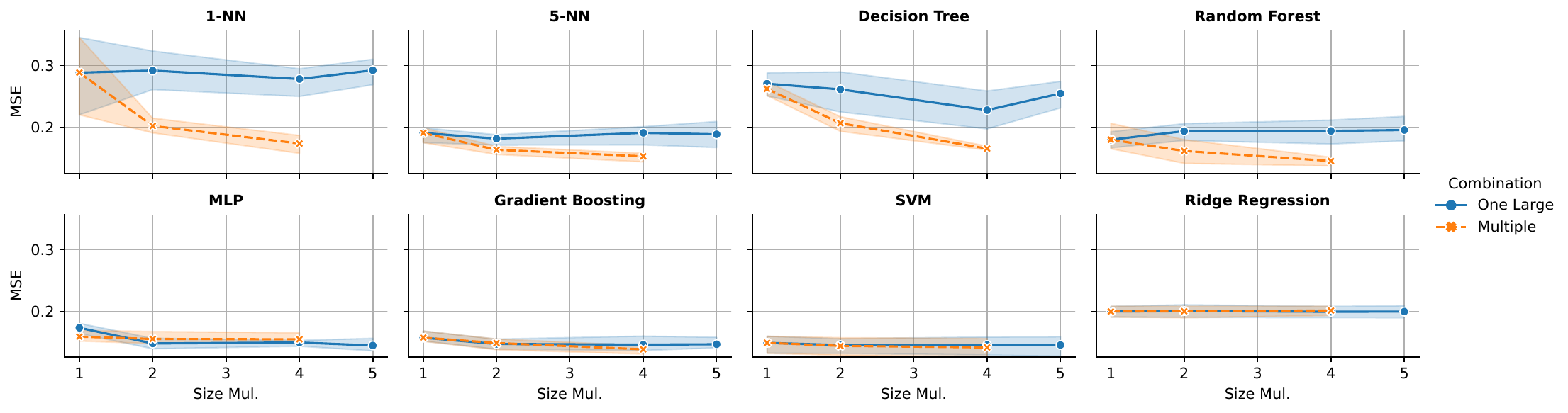}
        \caption{Insurance}
        \label{fig:one-large-regression-insurance}
    \end{subfigure}
    \caption{
        Comparison of the ensemble of multiple synthetic datasets with one large 
        synthetic dataset with an equal number of synthetic datapoints on the 
        regression datasets with synthpop. We see that 
        multiple synthetic datasets are always equal or better. Size Mul. 
        is the relative total number of synthetic data points 
        to the real data: for multiple synthetic datasets, it is $m$, and 
        for a single large synthetic dataset it is $n_{Syn} / n_{Real}$.
    }
    \label{fig:one-large-results-regression}
\end{figure*}

\end{document}